%% file: main.tex
\documentclass[10pt,journal,compsoc]{IEEEtran}

\usepackage[utf8]{inputenc} 
\usepackage[T1]{fontenc}    
\usepackage{url}            
\usepackage{booktabs}       
\usepackage{amsfonts}       
\usepackage{amsmath}
\usepackage{nicefrac}       
\usepackage{microtype}      
\usepackage{amsthm}
\usepackage{booktabs}
\usepackage{graphicx}
\usepackage{caption}
\usepackage{subcaption}
\usepackage{wrapfig}
\usepackage{bm}
\usepackage{multirow}
\usepackage{mathtools}
\usepackage{lscape}
\usepackage{hyperref}
\usepackage{url}
\usepackage{bigstrut}

\newcommand{\eat}[1]{}
\newtheorem{theorem}{Theorem}

\newtheorem{definition}{Definition}

\newtheorem{remark}{Remark}

\input{math_commands.tex}

\graphicspath{{./}{./figures/}}
\usepackage{xcolor}
\definecolor{mypink}{cmyk}{0, 0.7808, 0.4429, 0.1412}

\ifCLASSOPTIONcompsoc
  \usepackage[nocompress]{cite}
\else
  \usepackage{cite}
\fi
\ifCLASSINFOpdf
\else
\fi
\begin{document}
%
\title{Tackling Over-Smoothing for General Graph Convolutional Networks}
%
%
%
%

\author{Wenbing Huang$^\ast$, Yu Rong$^\ast$, Tingyang Xu, Fuchun Sun, Junzhou Huang
\IEEEcompsocitemizethanks{
\IEEEcompsocthanksitem Wenbing Huang (hwenbing@126.com) and Fuchun Sun (fcsun@mail.tsinghua.edu.cn) are with Beijing National Research Center for Information Science and Technology (BNRist), State Key Lab on Intelligent Technology and Systems, Department of Computer Science and Technology, Tsinghua University. \protect\\
\IEEEcompsocthanksitem Yu Rong (yu.rong@tencent.com), Tingyang Xu (tingyangxu@tencent.com), and Junzhou Huang (joehhuang@tencent.com) are with Tencent AI Lab, Shenzhen, China.\protect\\
\IEEEcompsocthanksitem Wenbing Huang and Yu Rong contributed to this work equally. \protect \\
\IEEEcompsocthanksitem Fuchun Sun is the corresponding author. \protect \\
}
}

%
%

\markboth{IEEE Transactions on Pattern Analysis and Machine Intelligence,~Vol.~14, No.~8, August~2015}%
{Shell \MakeLowercase{\textit{et al.}}: Bare Demo of IEEEtran.cls for Computer Society Journals}
%



\IEEEtitleabstractindextext{%
\begin{abstract}
Increasing the depth of GCN, which is expected to permit more expressivity, is shown to incur performance detriment especially on node classification. The main cause of this lies in over-smoothing. The over-smoothing issue
drives the output of GCN towards a space that contains limited distinguished information among nodes, leading to poor expressivity. Several works on refining the architecture of deep GCN have been proposed, but it is still unknown in theory whether or not these refinements are able to relieve over-smoothing. In this paper, we first theoretically analyze how general GCNs act with the increase in depth, including generic GCN, GCN with bias, ResGCN, and APPNP. We find that all these models are characterized by a universal process: all nodes converging to a cuboid. Upon this theorem, we propose DropEdge to alleviate over-smoothing by randomly removing a certain number of edges at each training epoch. Theoretically, DropEdge either reduces the convergence speed of over-smoothing or relieves the information loss caused by dimension collapse. Experimental evaluations on simulated dataset have visualized the difference in over-smoothing between different GCNs. Moreover, extensive experiments on several real benchmarks support that DropEdge consistently improves the performance on a variety of both shallow and deep GCNs. 
\end{abstract}

}

\maketitle

\IEEEdisplaynontitleabstractindextext

%
\IEEEpeerreviewmaketitle

\IEEEraisesectionheading{\section{Introduction}\label{sec:introduction}}
\IEEEPARstart{P}lenty of data are in the form of graph structures, where a certain number of nodes are irregularly related via edges. Examples include social networks~\cite{Kipf2017}, knowledge bases~\cite{ren2019query2box}, molecules~\cite{duvenaud2015convolutional}, scene graphs~\cite{xu2017scene}, etc.
Learning on graphs is crucial, not only for the analysis of the graph data themselves, but also for general data forms as graphs deliver strong inductive biases to enable relational reasoning and combinatorial generalization~\cite{battaglia2018relational}.
Recently, Graph Neural Network (GNN)~\cite{Wu2019} has become the most desired tool for the purpose of graph learning. The initial motivation of inventing GNNs is to generalize the success of Neural Networks (NNs) from tabular/grid data to the graph domain.


The key spirit in GNN is that it exploits recursive neighborhood aggregation function to combine the feature vector from a node as well as its neighborhoods until a fixed number of iterations $d$ (\emph{a.k.a.} network depth). Given an appropriately defined aggregation function, such message passing is proved to capture the structure around each node within its $d$-hop neighborhoods, as powerful as the Weisfeiler-Lehman (WL) graph isomorphism test~\cite{weisfeiler1968wltest} that is known to distinguish a broad class of graphs~\cite{xu2018powerful}. In this paper, we are mainly concerned with Graph Convolutional Networks (GCNs)~\cite{bruna2013spectral,Kipf2017,chen2018fastgcn,hamilton2017inductive,Klicpera2019,Huang2018,xu2018representation}, a central family of GNN that extends the convolution operation from images to graphs. GCNs have been employed successfully for the task of node classification which is the main focus of this paper.

As is already well-known in vision, the depth of Convolutional Neural Network (CNN) plays a crucial role in performance. 
Inspired from the success of CNN, one might expect to enable GCN with more expressivity to characterize richer neighbor topology by stacking more layers. 
Another reason of developing deep GCN stems from that characterizing graph topology requires sufficiently deep architectures. The works by~\cite{dehmamy2019understanding} and~\cite{loukas2019graph} have shown that GCNs are unable to learn a graph moment or estimate certain graph properties if the depth is restricted.

However, the expectation of formulating deep and expressive GCN is \textbf{not} easy to meet. This is because deep GCN actually suffers from the detriment of expressive power mainly caused by \emph{over-smoothing}~\cite{Li2018}.
An intuitive notion of over-smoothing is that the mixture of neighborhood features by graph convolution drives the output of an infinitely-deep GCN towards a space that contains limited distinguished information between nodes. From the perspective of training, over-smoothing erases important discriminative information from the input, leading to pool trainablity. We have conducted an example experiment in Figure~\ref{fig.compare}, in which the training of a deep GCN is observed to converge poorly. 

Several attempts have been proposed to explore how to build deep GCNs~\cite{Kipf2017,Xu2018,Klicpera2019,li2019can}. Nevertheless, none of them delivers sufficiently expressive architecture, and whether or not these architectures are theoretically guaranteed to prevent (or at least relieve) over-smoothing is still unclear. Li \emph{et al.} initially  linearized GCN as Laplacian smoothing and found that the features of vertices within each connected component of the graph will converge to the same values.
Putting a step forward from~\cite{Li2018}, Oono \& Suzuki~\cite{oono2019asymptotic} took both non-linearity (ReLU function) and convolution filters into account, and proved GCN converges to a subspace formulated with the bases of node degrees, but this result is limited to generic GCN~\cite{Kipf2017} without discussion of other architectures.

Hence, it remains open to answer, \emph{why and when, in theory, does over-smoothing happen for a general family of GCNs?} and \emph{can we, to what degree, derive a general mechanism to address over-smoothing and recover the expressive capability of deep GCNs?}

\begin{figure}[t!]
\centering
\includegraphics [width=0.24\textwidth]{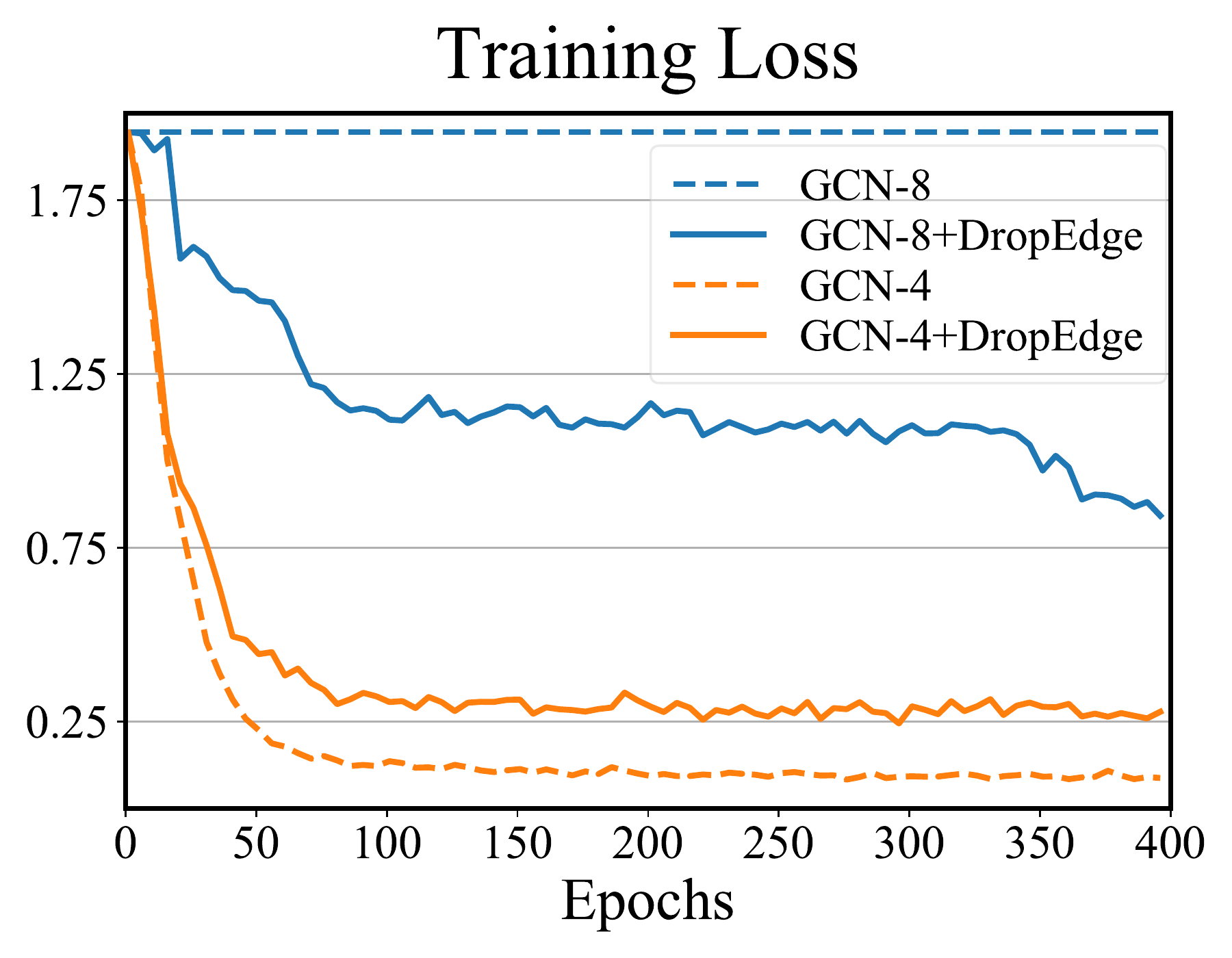}
\includegraphics [width=0.24\textwidth]{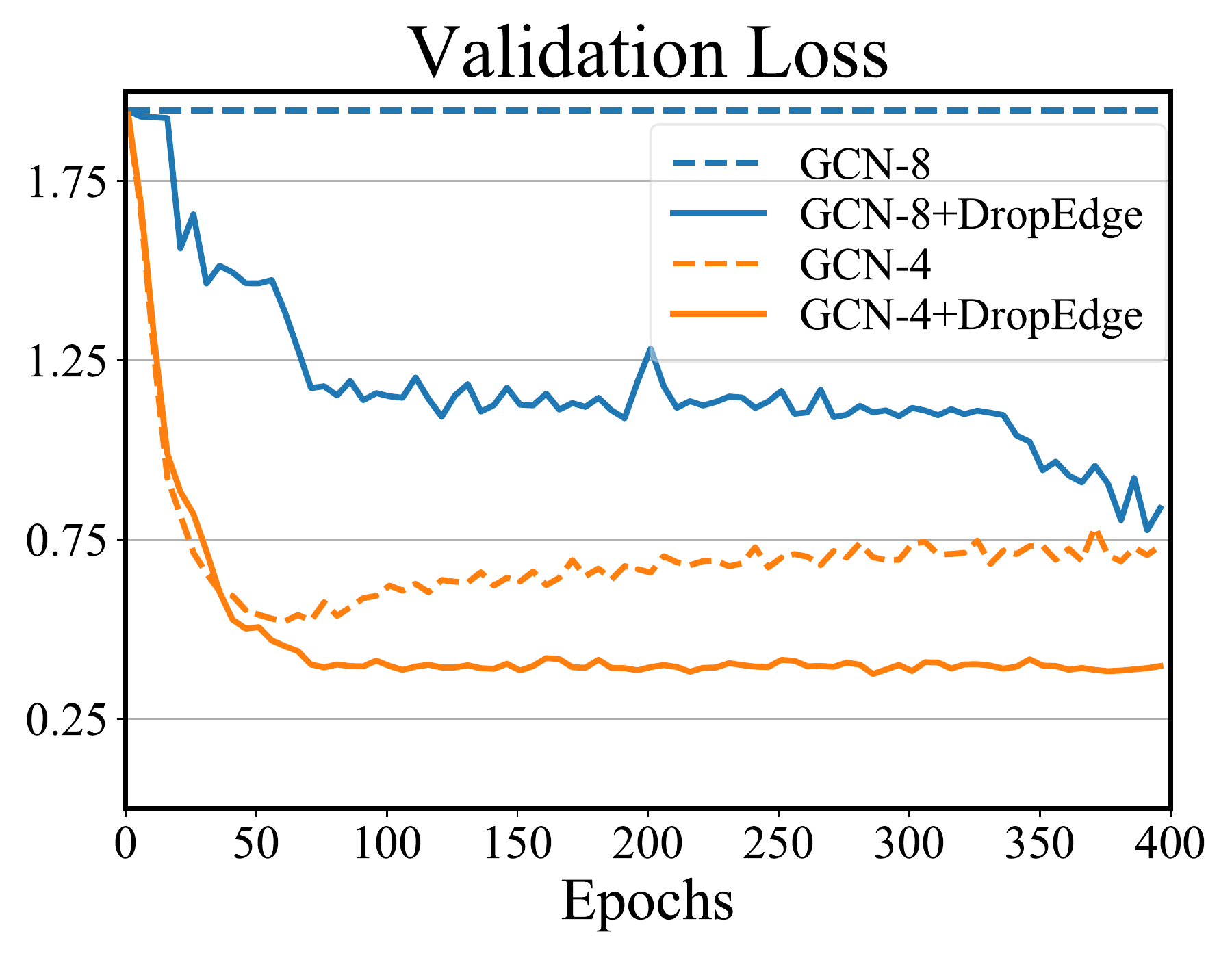}
\vskip -0.1in
\caption{Performance of GCNs on Cora.
We implement 4-layer and 8-layer GCNs w and w/o DropEdge. GCN-4 gets stuck in the over-fitting issue attaining low training error but high validation error; the training of GCN-8 fails to converge satisfactorily due to over-smoothing. By applying DropEdge, both GCN-4 and GCN-8 work well for both training and validation. Note that GCNs here have no bias.}
\label{fig.compare}
\end{figure}

To this end, we first revisit the concept of over-smoothing in a general way. Besides generic GCN~\cite{Kipf2017}, we explore GCN with bias~\cite{dehmamy2019understanding} that is usually implemented in practice, ResGCN~\cite{Kipf2017} and APPNP~\cite{Klicpera2019} that refine GCN by involving skip connections. We mathematically prove, if we go with an infinite number of layers, all these models will converge to a \emph{cuboid} that expands the subspace proposed by~\cite{oono2019asymptotic} up to a certain radius $r$. Such theoretical finding is interesting and refreshes current results by~\cite{Li2018, oono2019asymptotic} in several aspects. First, converging to the cuboid implies converging to the subspace, but not vice verse. Second, unlike existing methods~\cite{Li2018, oono2019asymptotic} that focus on GCN without bias, our conclusion shows that adding the bias leads to non-zero radius, which, interestingly, will somehow impede over-smoothing. Finally, our theorem suggests that ResGCN slows down over-smoothing and APPNP always maintains certain input information, both of which are consistent with our instinctive understandings, yet not rigorously explored before.

Over-smoothing towards a cuboid rather than a subspace, albeit not that bad, still restricts expressive power and requires to be alleviated.
In doing so, we propose DropEdge.  The term ``DropEdge'' refers to randomly dropping out certain rate of edges of the input graph for each training time. In its particular form, each edge is independently dropped with a fixed probability $p$, with $p$ being a hyper-parameter and determined by validation. There are several benefits in applying DropEdge for the GCN training (see the experimental improvements by DropEdge in Fig.~\ref{fig.compare}). First, DropEdge can be treated as a message passing reducer. In GCNs, the message passing between adjacent nodes is conducted along edge paths. Removing certain edges is making node connections more sparse, and hence avoiding over-smoothing to some extent when GCN goes very deep. Indeed, as we will draw theoretically in this paper, DropEdge either slows down the degeneration speed of over-smoothing or reduces information loss. 

Anther merit of DropEdge is that it can be considered as a data augmentation technique as well. By DropEdge, we are actually generating different random deformed copies of the original graph; as such, we augment the randomness and the diversity of the input data, thus better capable of preventing over-fitting. It is analogous to performing random rotation, cropping, or flapping for robust CNN training in the context of images.
Note that DropEdge is related to  the random graph generation methods (such as the Erdos-Renyi (ER) model~\cite{erdos1960evolution} or sparse graph learning approaches (such as GLASSO~\cite{friedman2008sparse}) in terms of edge modification. Nevertheless, DropEdge removes different subset of edges for different training iteration according to the uniform distribution, while the ER model or GLASSO  employ the same graph for all training iterations once the random/sparser graph is created.
This is why DropEdge is able to alleviate over-smoothing for each training iteration, but,  during the whole training phrase, it still preserves the full information of  the  underlying  graph  kernel in a probabilistic sense. 

We provide a complete set of experiments to verify our conclusions related to our rethinking on over-smoothing and the efficacy of DropEdge on four benchmarks of node classification. In particular, our DropEdge---as a flexible and general technique---is able to enhance the performance of various popular backbone networks, including GCN~\cite{Kipf2017}, ResGCN~\cite{li2019can}, JKNet~\cite{Xu2018}, and APPNP~\cite{Klicpera2019}. It demonstrates that DropEdge consistently improves the performance on a variety of both shallow and deep GCNs.
Complete details are provided in \textsection~\ref{sec:exps}.

To sum up, our contributions are as follows.
\begin{itemize}
    \item We study the asymptotic behavior for the output of general deep GCNs (generic GCN, GCN with bias, ResGCN, and APPNP) {with the involvement of non-linearity}. We theoretically show that these GCNs will converge to a cuboid with infinite layers stacked. 
    \item We propose DropEdge, a novel technique that uniformly drops a certain number of edges during each training iteration, which is proved to relieve the over-smoothing issue of general deep GCNs in terms of slowing down the convergence speed or decreasing the information loss.
    \item Experiments on four node classification benchmarks support the rationality of our proposed theorems and indicates that DropEdge is able to enhance a variety of GCNs for both shallow and deep variants.  
\end{itemize}

\section{Related Work}
\textbf{GCNs.}
The first prominent research on GCNs is presented in \cite{bruna2013spectral}, which develops graph convolution based on both the spectral and spatial views. Later, \cite{Kipf2017,defferrard2016convolutional,henaff2015deep,Li2018a,Levie2017} apply improvements, extensions, and approximations on spectral-based GCNs. To address the scalability issue of spectral-based GCNs on large graphs, spatial-based GCNs have been rapidly developed~\cite{hamilton2017inductive,Monti2017,niepert2016learning,Gao2018}. Recently, several sampling-based methods have been proposed for fast graph representation learning, including the node-wise sampling methods~\cite{hamilton2017inductive}, the layer-wise approaches~\cite{chen2018fastgcn,Huang2018}, and the graph-wise methods~\cite{chiang2019clustergcn,zeng2020graphsaint}. Specifically, GAT~\cite{DBLP:journals/corr/abs-1710-10903} has discussed applying dropout on edge attentions. While it actually is a post-conducted version of DropEdge before attention computation, the relation to over-smoothing is never explored in~\cite{DBLP:journals/corr/abs-1710-10903}. In our paper, however, we have formally presented the formulation of DropEdge and provided rigorous theoretical justification of its benefit in alleviating over-smoothing.

\textbf{Deep GCNs.}
Despite the fruitful progress, most previous works only focus on shallow GCNs while the deeper extension is seldom discussed. The attempt for building deep GCNs is dated back to the  GCN paper~\cite{Kipf2017}, where the residual mechanism is applied; unexpectedly, as shown in their experiments, residual GCNs still perform worse when the depth is 3 and beyond.
The authors in~\cite{Li2018} first point out the main difficulty in constructing deep networks lying in over-smoothing, but unfortunately, they never propose any method to address it. The follow-up study~\cite{Klicpera2019} solves over-smoothing by using personalized PageRank that additionally involves the rooted node into the message passing loop.
JKNet~\cite{Xu2018} employs dense connections for multi-hop message passing which is compatible with DropEdge for formulating deep GCNs. Recently, DAGNN~\cite{liu2020towards} refines the architecture of GCN by first decoupling the representation transformation from propagation and then utilizing an adaptive adjustment mechanism to balance the information from local and global neighborhoods for each node.

The authors in~\cite{oono2019asymptotic} theoretically prove that the node features of deep GCNs will converge to a subspace and incur information loss. It generalizes the conclusion in~\cite{Li2018} by considering the ReLU function and convolution filters. In this paper, we investigate the over-smoothing behaviors of a broader class of GCNs, and show that general GCNs will converge to a cuboid other than a subspace. Chen et al.~\cite{chen2020measuring} develop a measurement of over-smoothing based on the conclusion of~\cite{Li2018} and propose to relieve over-smoothing by using a supervised optimization-based method, while our DropEdge is proved to alleviate general GCNs by just random edge sampling, which is simple yet effective.  Other recent studies to prevent over-smoothing resort to activation normalization~\cite{zhao2019pairnorm} and doubly residual connections~\cite{chensimple}, which are complementary with our DropEdge.
A recent method ~\cite{li2019can} has incorporated residual layers, dense connections and dilated convolutions into GCNs to facilitate the development of deep architectures,
where over-smoothing is not discussed. 

{
\textbf{Other related works.}
In DropEdge, the idea of removing edges from the input graph is similar but distinct from the sparse graph learning methods~\cite{friedman2008sparse,egilmez2016graph}.  By DropEdge, we are NOT implying that the edges of the underlying graph kernel are sufficient uninformative; instead, DropEdge still preserves this kind of information, as it acts \textbf{in a random yet unbiased way}. We will provide more explanations in \textsection~\ref{sec:methodology} and evaluations in \textsection~\ref{sec:glasso}. The well-known Perron-Fronbenius Theorem (PFT)~\cite{pillai2005perron} and spectral graph theory~\cite{chung1997spectral} have characterized the convergence behavior of random walks on graphs. However, these results are not directly applicable for our case as the adjacency here is augmented with self-loops, and the model is more complicated beyond the random walk, \emph{e.g.} with the involvement of non-linearity. 

}

\section{Preliminaries}

\subsection{Graph denotations and the spectral analysis.}
{
Let $\gG=(\sV, \mathcal{E})$ represent the input graph of size $N$ with nodes $v_i\in\sV$ and edges $(v_i, v_j)\in\mathcal{E}$. We denote by $\mX=\{\vx_1,\cdots,\vx_N\}\in\R^{N\times C}$ the node features, and by $\mA\in\R^{N\times N}$ the adjacency matrix where the element $\mA(i,j)$ returns the weight of each edge $(v_i, v_j)$. The node degrees are given by $\vd=\{d_1,\cdots,d_N\}$ where $d_i$ computes the sum of edge weights connected to node $i$. We define $\mD$ as the degree matrix whose diagonal elements are obtained from $\vd$. Following the previous researches~\cite{Kipf2017,Xu2018,Klicpera2019}, the edge weights are supposed to be non-negative and only capture the similarity between nodes instead of their negative correlations. 
}

As we will introduce later, GCN~\cite{Kipf2017} applies the normalized augmented adjacency by adding self-loops followed by augmented degree normalization, which results in $\hat{\mA}=\hat{\mD}^{-1/2}(\mA+\mI)\hat{\mD}^{-1/2}$, where $\hat{\mD}=\mD+\mI$. We define the augmented normalized Laplacian~\cite{oono2019asymptotic} as $\hat{\mL}=\mI-\hat{\mA}$. By setting up the relation with the spectral theory of generic Laplacian~\cite{chung1997spectral}, Oono \& Suzuki~\cite{oono2019asymptotic} derive the spectral for the augmented normalized Laplacian and its adjacency thereby. We summarize the result as follows.

\begin{theorem}[Augmented Spectral Property~\cite{oono2019asymptotic}]
\label{th:spectral}
Since $\hat{\mA}$ is symmetric, let $\lambda_1\le\cdots\le\lambda_N$ be the real eigenvalues of $\hat{\mA}$, sorted in an ascending order. Suppose the multiplicity of the largest eigenvalue $\lambda_N$ is $M$, \emph{i.e.}, $\lambda_{N-M}<\lambda_{N-M+1}=\cdots=\lambda_N$. Then we have:
\begin{itemize}
    \item $-1<\lambda_1,\lambda_{N-M}<1$;
    \item $\lambda_{N-M+1}=\cdots=\lambda_N=1$;
    \item $M$ is given by the number of connected components in $\gG$, and  $\hat{\ve}_m\coloneqq\hat{\mD}^{1/2}\vu_m$ is the eigenvector associated with eigenvalue $\lambda_{N-M+m}$ where $\vu_m\in\R^{N}$ is the indicator of the $m$-th connected component, \emph{i.e.}, $\vu_m(i)=1$ if node $i$ belongs to the $m$-th component and $\vu_m(i)=0$ otherwise.
\end{itemize}
\end{theorem}
{
Theorem 1 focuses on the eigenvalues of the the normalized augmented adjacency $\hat{\mA}$. The well-known Perron-Fronbenius Theorem (PFT)~\cite{pillai2005perron} states that for a row-stochastic, irreducible probability matrix $\mP$, the spectral radius is exactly 1, and all other eigenvalues have magnitude strictly smaller than 1. Yet, this result can not be applied directly since $\hat{\mA}$ is not a irreducible stochastic matrix. The conclusion by~\cite{chung1997spectral} is applicable for the spectral analysis of the adjacency matrix $\mA$, but it has not taken the augmented self-loops (\emph{i.e.} $\mA+\mI$) into account. Theorem 1 rigorously generalizes the conclusion by~\cite{chung1997spectral} from $\mA$ to its augmented version $\hat{\mA}$. 
}

\subsection{Variants of GCN}
\label{sec:gcns}
Here, we introduce several typical variants of GCN.

\textbf{Generic GCN.} 
As originally developed by~\cite{Kipf2017}, the feed forward propagation in GCN is recursively conducted as
\begin{eqnarray}
\label{Eq:gcn}
\mH_{l+1} &=& \sigma\left(\hat{\mA}\mH_{l}\mW_{l}\right),
\end{eqnarray}
where $\mH_{l}=\{\vh_{1,l},\cdots,\vh_{N,l}\}$ are the hidden vectors of the $l$-th layer with $\vh_{i,l}$ being the hidden feature for node $i$; $\sigma(\cdot)$ is a nonlinear function (it is implemented as ReLU throughout this paper); and $\mW_{l}\in\mathbb{R}^{C_l \times C_{l+1}}$ is the filter matrix in the $l$-th layer. For the analyses in \textsection~\ref{sec:our-methods}, we set the dimensions of all layers to be the same $C_l=C$ for simplicity. We henceforth call generic GCN as GCN for short unless otherwise specified. 

\textbf{GCN with bias (GCN-b).}
In most literature, GCN is introduced in the form of Eq.~\ref{Eq:gcn} without the explicit involvement of the bias term that, however, is necessary in practical implementation. If adding the bias, Eq.~\ref{Eq:gcn} is renewed as
\begin{eqnarray}
\label{Eq:gcn-b}
\mH_{l+1} &=& \sigma\left(\hat{\mA}\mH_{l}\mW_{l}+\vb_{l}\right),
\end{eqnarray}
where the bias is defined by $\vb_l\in\R^{1\times C}$.

\textbf{ResGCN.}
By borrowing the concept from ResNet~\cite{he2016deep}, Kipf \& Welling~\cite{Kipf2017} utilize residual connections between hidden layers to facilitate the training of deeper models by carrying over information from the previous layer’s input: 
\begin{eqnarray}
\label{Eq:resgcn}
\mH_{l+1} &=& \sigma\left(\hat{\mA}\mH_{l}\mW_{l}\right)+\alpha\mH_l,
\end{eqnarray}
where we have further added the weight $0\leq\alpha\leq1$ for more flexibility to balance between the GCN propagation and residual information. 

\textbf{APPNP.}
Since deep GCNs will isolate the output from the input due to over-smoothing, Klicpera et al.~\cite{Klicpera2019} suggest to explicitly conduct skip connections from the input layer to each hidden layer to preserve input information:
\begin{eqnarray}
\label{Eq:appnp}
\mH_{l+1} &=& (1-\beta)\hat{\mA}\mH_{l}+\beta\mH_0,
\end{eqnarray}
where $0\leq\beta\leq1$ is the trade-off weight.
Note that the original version by~\cite{Klicpera2019} dose not involve the non-linearity and weight matrix in each hidden layer. The work by GCNII~\cite{chensimple} seeks for more capacity by adding the ReLU function and trainable weights to the propagation. Here we adopt the original version and find it works promisingly.


\begin{figure}[t!]
\centering
\includegraphics [width=.45\textwidth]{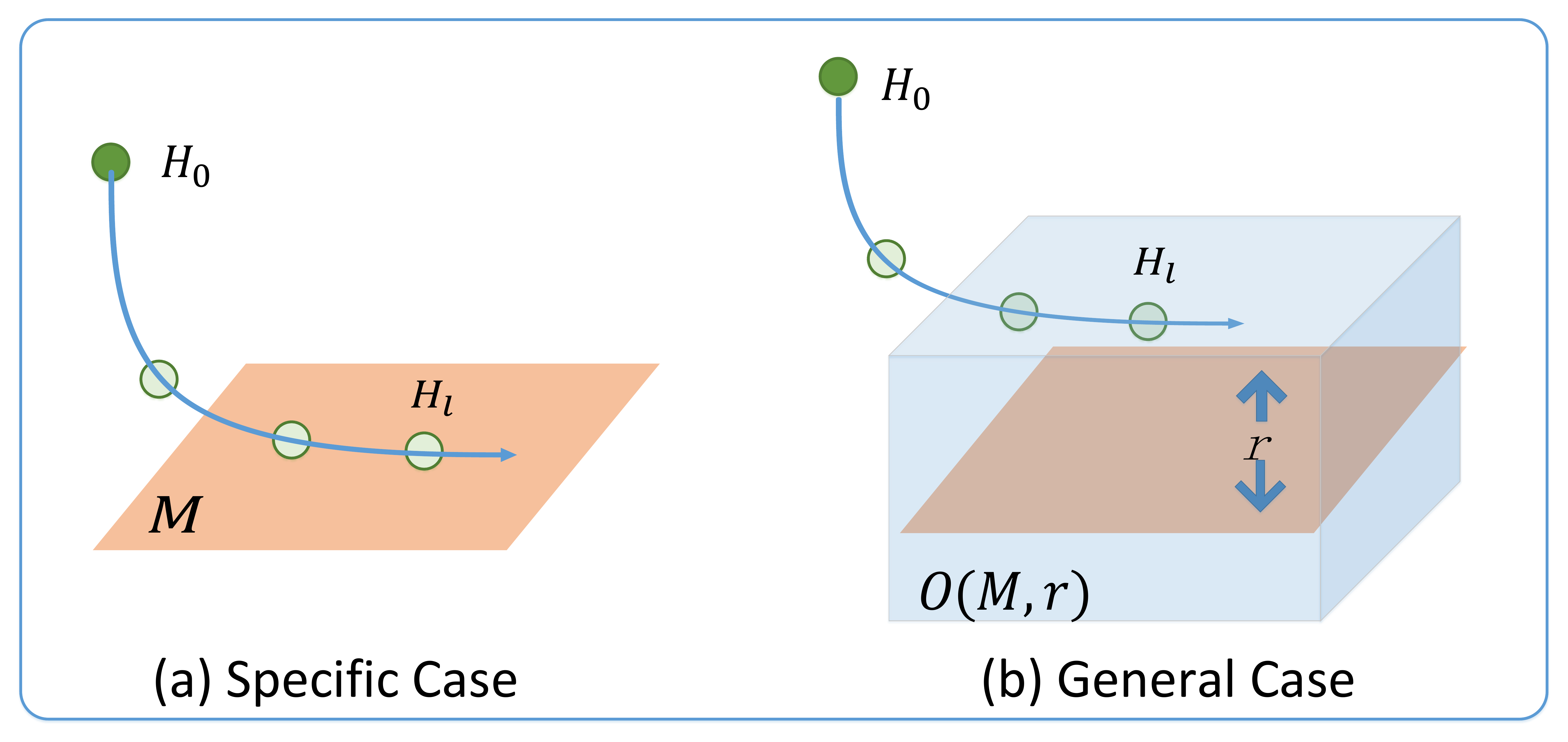}
\caption{(a) Over-smoothing of generic GCN~\cite{oono2019asymptotic}: converging to a multi-dimensional subspace $\gM$; (b) Over-smoothing of general models: converging to a cuboid $O(\gM,r)$. }
\label{fig.illustration}
\end{figure}

\section{Analyses and Methodologies}
\label{sec:our-methods}
In this section, we first derive the universal theorem (\autoref{th:universal}) to explain why and when over-smoothing will happen for all the four models introduced in \textsection~\ref{sec:gcns}. We then introduce DropEdge that is proved to relieve over-smoothing for all models. We also contend that our DropEdge is able to prevent over-fitting, and involve the discussions and extensions of DropEdge with other related notions.

\subsection{Over-smoothing of General Models}
\label{sec:analysis}

The notion of ``over-smoothing'' is originally introduced by ~\cite{Li2018}, later explained by~\cite{oono2019asymptotic} and many other recent works~\cite{rong2019dropedge,chen2020measuring,zhao2019pairnorm,chensimple}. 
{
In general, the over-smoothing phenomenon implies that the node representations become mixed and in-distinguishable with each other after many layers of message passing. It hence weakens the trainability and expressivity of deep GCNs. Notice that over-smoothing in our context is related to ``frequency smoothing" in classical signal processing, if frequency is understood as the eigenvalue of the adjacency/Laplacian matrix, joining the definition by~\cite{ortega2018graph}. Over-smoothing can be explained as graph frequency filtering of small eigenvalues. However, compared to the conventional studies in frequency smoothing~\cite{ortega2018graph}, the analysis of over-smoothing in GCN models is more challenging given the complicated architecture and the involvement of non-linearity, which motivates the theoretical studies in this paper. }

In our following analyses, we exploit the conclusion by~\cite{oono2019asymptotic} for its generality of taking both the non-linearity (\emph{i.e.} the ReLU function) and the convolution filters into account. The authors in~\cite{oono2019asymptotic} explain over-smoothing of deep GCN as convergence to a multi-dimensional subspace. Falling into this subspace will encounter information loss: the nodes within the same connected component are distinguishable only by their degrees. In other words, if two nodes share the same degree in the same component, their representations will be the same after infinite-layer propagation, even their initial features and local topology are clearly different. Such information loss will become more serious if the number of connected components, \emph{i.e.} the dimensionality of the subspace (defined below) is small. We can hence leverage the distance between each GCN layer and the subspace to measure how serious the over-smoothing is. 

The definition of the subspace is given below.
\begin{definition}[Subspace]
\label{de:subspace}
We define $\gM\coloneqq\{\mH\in\R^{N\times C}|\mH=\hat{\mE}\mC, \mC\in\R^{M\times C}\}$ as an $M$-dimensional ($M\le N$) subspace in $\R^{N\times C}$, where $\hat{\mE}=\{\hat{\ve}_1,\cdots,\hat{\ve}_M\}\in\R^{N\times M}$ collects the bases of the largest eigenvalue of $\hat{\mA}$ in \autoref{th:spectral}, namely, $\hat{\ve}_m=\hat{\mD}^{1/2}\vu_m$. 
\end{definition}

{
The subspace $\gM$ is expanded by the columns of $\hat{\mE}$. Definition~\ref{de:subspace} follows the conventional definition of a subspace in~\cite{vetterli2014foundations}, namely, $\gM$ is closed under addition and scalar multiplication. To be specific, if $\mH_1,\mH_2\in\gM$, for any $a, b\in\R$, $a\mH_1 + b\mH_2 = a\hat{\mE}\mC_1 + b\hat{\mE}\mC_2 = \hat{\mE}(a\mC_1+b\mC_2)=\hat{\mE}\mC\in\gM$, since $\mC\coloneqq a\mC_1+b\mC_2\in\R^{M\times C}$.
} We define the distance between matrix $\mH\in\R^{N\times M}$ and $\gM$ as $d_\mathcal{M}(\mH)\coloneqq\inf_{\mY\in \mathcal{M}} ||\mH-\mY||_\mathrm{F}$, where $\|\cdot\|_F$ computes the Frobenius norm.
However, the conclusion by~\cite{oono2019asymptotic} is only applicable for generic GCN. In this section, we derive a universal theorem to characterize the behavior of general GCNs, showing that they will converge to a cuboid other than a subspace with the increase of depth. We first define the cuboid below.

\begin{definition}[Cuboid]
\label{de:cuboid}
We define $\gO(\gM,r)$ as the cuboid that expands $\gM$ up to a radius $r\geq 0$, namely, $\gO(\gM,r)\coloneqq\{\mH\in\R^{N\times C}|d_{\gM}(\mH)\leq r\}$.
\end{definition}

We now devise the general theorem on over-smoothing.
\begin{theorem}[General Over-Smoothing Theorem]
\label{th:universal}
For the GCN models defined in Eq.~\ref{Eq:gcn} to Eq.~\ref{Eq:appnp}, we universally have
\begin{equation}
    d_\mathcal{M}(\mH_{l+1})-r \le v\left(d_\mathcal{M}(\mH_{l})-r\right), \label{equ:distance-b}
\end{equation}
where $v\geq0$ and $r$ describe the convergence factor and radius, respectively, depending on what the specific model is. In particular,
\begin{itemize}
    \item For generic GCN (Eq.~\ref{Eq:gcn}), $v=s\lambda$, $r=0$;
    \item For GCN-b (Eq.~\ref{Eq:gcn-b})\footnote{We assume the distance $d_\mathcal{M}(\vb_l)$ keeps the same for all layers for simplicity; otherwise, we can define it as the supremum.}, $v=s\lambda$, $r=\frac{d_\mathcal{M}(\vb_l)}{1-v}$;
    \item For ResGCN (Eq.~\ref{Eq:resgcn}), $v=s\lambda+\alpha$, $r=0$;
    \item For APPNP (Eq.~\ref{Eq:appnp}), $v=(1-\beta)\lambda$, $r=\frac{\beta d_\mathcal{M}(\mH_0)}{1-v}$,
\end{itemize}
where, $s>0$ is the supremum of all singular values of all $\mW_l$, and $\lambda\coloneqq \max_{n=1}^{N-M}|\lambda_n|<1$ is  
the second largest eigenvalue of $\hat{\mA}$. The equality in Eq.~\ref{equ:distance-b} is achievable under certain specification.
\end{theorem}

{
	The main characteristic of Theorem 2, in contrast to the power method~\cite{van1996matrix} and random-walk-based methods~\cite{Li2018}, is that it has taken the non-linear ReLU function $\sigma$ into account. By involving the non-linearity, it is indeed challenging to analyze the convergence behavior, which resorts to certain tricky transformations and inequations, as demonstrated by Eq.~17-20 in Appendix. In addition, Theorem 2 is applicable for general models, making it more powerful than the methods~\cite{oono2019asymptotic} that are proposed for certain specific case.}

The proof is provided in Appendix A. By Eq.~\ref{equ:distance-b}, we recursively derive $d_{\gM}(\mH_{l})-r\leq v (d_{\gM}(\mH_{l-1})-r)\leq\cdots\leq v^l(d_{\gM}(\mH_{0})-r)$.
We assume $v<1$ for any $v\in\{s\lambda, s\lambda+\alpha, (1-\beta)\lambda\}$ in \autoref{th:universal} by observing that $\lambda<1$, $s\leq1$ (which is usually the case due to the $\ell_2$ penalty during training) and $\alpha$ can be set to small enough\footnote{Otherwise, if $v\geq1$, it will potentially cause gradient explosion and unstable training for deep GCNs, which is not the focus of this paper.}. Under this assumption, \autoref{th:universal} states that the general GCN models actually converge towards the cuboid $O(\gM,r)$, as depicted in Fig.~\ref{fig.illustration}.
We further analyze the convergence behavior for each particular model with the following remarks.

\begin{remark}
\label{re:gcn}
For generic GCN without bias, the radius becomes $r=0$, and we have  $\lim_{l\rightarrow\infty}d_{\gM}(\mH_{l+1})\leq\lim_{l\rightarrow\infty}v^l d_{\gM}(\mH_{0})=0$, indicating that $\mH_{l+1}$ exponentially converges to $\gM$ and thus results in over-smoothing, as already studied by~\cite{oono2019asymptotic}.
\end{remark}

\begin{remark}
\label{re:gcn-b}
For GCN-b, the radius is not zero: $r>0$, and we have 
$\lim_{l\rightarrow\infty}d_{\gM}(\mH_{l+1})\leq \lim_{l\rightarrow\infty}r+v^l( d_{\gM}(\mH_{0})-r)=r$, \emph{i.e.}, $\mH_{l+1}$ exponentially converges to the cuboid $O(\gM,r)$. Unlike $\gM$, $O(\gM,r)$ shares the same dimensionality with $\R^{N\times C}$ and probably contains useful information (other than node degree) for node representation. 
\end{remark}

\begin{remark}
\label{re:resgcn}
For ResGCN, it finally converges to $\gM$ similar to generic GCN. Yet, as $v=s\lambda+\alpha \geq s\lambda$, it exhibits slower convergence speed to $\gM$ compared to generic GCN, which is consistent with our intuitive understanding of the benefit by adding residual connections.
\end{remark}

\begin{remark}
\label{re:appnp}
For APPNP, it converges to $O(\gM,r)$ other than $\gM$ with $r>0$. This explains why adding the input layer to each hidden layer in APPNP helps impede over-smoothing. Notice that increasing $\beta$ will enlarge $r$ but decrease $v$ at the same time, thus leading to faster convergence to a larger cuboid.
\end{remark}


The discussions above in Remarks~\ref{re:gcn}-\ref{re:appnp} show that the value of $\lambda$ plays an important role in influencing over-smoothing for different models, larger $\lambda$ implying less over-smoothing. In the next subsection, we will introduce that our proposed method DropEdge is capable of increasing $\lambda$ and preventing over-smoothing thereby. 





\subsection{DropEdge to Alleviate Over-Smoothing}
\label{sec:methodology}
At each training epoch, the DropEdge technique drops out a certain rate of edges of the input graph by random. Formally, it randomly enforces $Vp$ non-zero elements of the adjacency matrix $\mathbf{A}$ to be zeros, where $V$ is the total number of edges and $p$ is the dropping rate. If we denote the resulting adjacency matrix as $\mA_{\text{drop}}$, then its relation with $\mA$ becomes
\begin{eqnarray}
\label{Eq:DropEdge}
\mA_{\text{drop}} &=& \text{Unif}(\mA, 1-p),
\end{eqnarray}
where $\text{Unif}(\mA, 1-p)$ uniformly samples each edge in $\mA$ with property $1-p$, namely, $\mA_{\text{drop}}(i,j)=\mA(i,j)*\text{Bernoulli}(1-p)$. In our implementation, to avoid redundant sampling edge, we create $\mA_{\text{drop}}$ by drawing a subset of edges of size $V(1-p)$ from $\mA$ in a non-replacement manner.   Following the idea of~\cite{Kipf2017}, we also perform the re-normalization trick on $\mA_{\text{drop}}$ to attain $\hat{\mA}_{\text{drop}}$. We replace $\hat{\mA}$ with $\hat{\mA}_{\text{drop}}$ in Eq.~\ref{Eq:gcn} for propagation and training. When validation and testing, DropEdge is not utilized.

\autoref{th:universal} tells that the degenerated expressivity of deep GCNs is closely related to $v$ and thereby $\lambda$---the absolute second-largest eigenvalue of $\hat{\mA}$. Here, we will demonstrate that adopting DropEdge decreases $\lambda$ and alleviates over-smoothing. In our previous conference version~\cite{rong2019dropedge}, we only discuss how DropEdge influences the spectral of $\mA_{\text{drop}}$ without considering the re-normalization trick. In this paper, we will draw the conclusion directly for the normalized augmented adjacency matrix $\hat{\mA}_{\text{drop}}$ below.
\begin{theorem}
\label{The:smoothing}
We denote $\lambda(p)$ as any absolute eigenvalue of the expected $\mA_{\text{drop}}$ under dropping rate $p$ after re-normalization. We can bound the value of $\lambda(p)$ by
\begin{eqnarray}
\label{eq:lambda_drop}
\mu(p)\leq\lambda(p)\leq\gamma(p),
\end{eqnarray}
where both $\mu(p)$ and $\gamma(p)$ are monotonically increasing functions with regard to $p$. Besides, the gap $\gamma(p)-\mu(p)$ monotonically decreases with respect to $p$, and when $p=1$, the gap reduces to zero, leading to $\mu(1)=\lambda(1)=\gamma(1)=1$.
\end{theorem}


We provide the full details in Appendix B.
Theorem~\ref{The:smoothing} tells that
performing DropEdge is able to increase both the upper and lower bounds of $\lambda$ (and decrease their gap at the same time), which will enforce $\lambda$ towards a larger value particularly when
$p$ is close to 1 (with sufficient edges dropped). 
This, to a certain degree, can slow down the over-smoothing speed in Eq.~\ref{equ:distance-b}. Fig.~\ref{fig.dropedge-os} illustrates the relation between $\lambda$ and $p$, where $\lambda$ may fluctuates initially but its value is finally increased when $p$ is enlarged. 

DropEdge is also able to increase the dimensionality of the subspace, thus alleviating information loss, as already proved by our conference paper~\cite{rong2019dropedge}. We summarize this property as a theorem as follows. 
\begin{theorem}
\label{The:smoothing-2}
Regarding the GCN models in Eq.~\ref{Eq:gcn} to Eq.~\ref{Eq:appnp}, we assume by $\gM$ the convergence subspace defined in Definition~\ref{de:subspace} on the original graph, and by $\gM'$ on the one after DropEdge. Then, after certain edges removed, the information loss is only decreased: $N-\text{dim}(\gM) \geq N-\text{dim}(\gM')$\footnote{In a general sense, the dimensionality of data space does not necessarily reflect the amount of information, but in this paper converging to the subspace of smaller dimension does indicates more serious information loss considering the particular structure of the subspace given by Definition~\ref{de:subspace}.}.
\end{theorem}



\begin{figure}[t!]
\centering
\includegraphics [width=.35\textwidth]{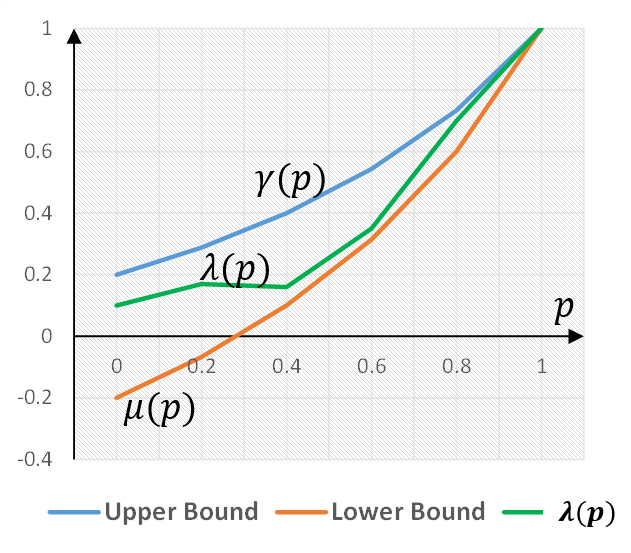}
\caption{Illustration of Eq.~\ref{eq:lambda_drop} where $\lambda$ is bounded by $\mu(p)$ and $\gamma(p)$. The derivations of $\mu(p)$ and $\gamma(p)$ are given in Appendix~B.}
\label{fig.dropedge-os}
\end{figure}

\autoref{The:smoothing}-\ref{The:smoothing-2} do suggest that DropEdge is able to alleviate over-smoothing, but they do \textbf{not} mean preventing over-smoothing by DropEdge will always deliver enhanced classification performance. For example, dropping all edges will address over-smoothing completely, which yet will weaken the model expressive power as well since the the GCN model has degenerated to an MLP without considering topology modeling. In general, how to balance between preventing over-smoothing and expressing graph topology is critical, and one should take care of choosing an appropriate edge dropping rate $p$ to reflect this. In our experiments, we select the value of $p$ by using validation data, and find it works well in a general way.
\begin{table*}[t!]
  \centering
  \caption{Dataset Statistics}
  \small
\begin{tabular}{lrrrrcl}
\hline
Datasets & \multicolumn{1}{l}{Nodes} & \multicolumn{1}{l}{Edges} & \multicolumn{1}{l}{Classes} & \multicolumn{1}{l}{Features} & \multicolumn{1}{l}{Traing/Validation/Testing} & Type \\
\hline
Cora  & 2,708 & 5,429 & 7     & 1,433 & 1,208/500/1,000 & Transductive \\
Citeseer & 3,327 & 4,732 & 6     & 3,703 & 1,812/500/1,000 & Transductive \\
Pubmed & 19,717 & 44,338 & 3     & 500   & 18,217/500/1,000 & Transductive \\
Reddit & 232,965 & 11,606,919 & 41    & 602   & 152,410/23,699/55,334 & Inductive \\
\hline
\end{tabular}%
  \label{table:data}%
\end{table*}%

{
We would like to highlight that DropEdge acts \textbf{in a random yet unbiased way}. It does drop a certain number of edges at each specific training iteration for relieving over-smoothing, but different edges are removed at different iteration in accordance with the uniform probability. In expectation, the information of the graph edges (and the underlying kernel) are still retained from the perspective of the whole training process. Briefly, the hidden feature of node $i$ in the $(l+1)$-th layer is aggregated from all its neighborhoods in the $l$-th layer, which means the full aggregation is defined as $\vh_{i, l+1}=\sum_{j\in\gN(i)}\mA(i,j)\vh_{j,l}$. In DropEdge, each edge is drawn from a Bernoulli distribution denoted as $\text{Bernoulli}(1-p)$, where $p$ is the dropping rate. Then, the expectation of the aggregation with regard to all edges is given by $\E[\vh_{i, l+1}|\mH_l]=\E[\sum_{j\in\gN(i)}\mA_{i,j}\vh_{j,l}|\mH_l]=\sum_{j\in\gN(i)}\E[\mA_{i,j}]\vh_{j,l}=(1-p)\sum_{j\in\gN(i)}\mA_{i,j}\vh_{j,l}$, which is the same as the original full aggregation up to a multiplier $1-p$, and this multiplier is erased if the sum-to-1 normalization is conducted on adjacency weights $\mA$. Such unbiased sampling behavior makes our DropEdge distinct from the random graph generation methods~\cite{erdos1960evolution} or sparse graph learning approaches~\cite{friedman2008sparse} where the graph once generated/modified will keep fixed for all training iterations, leading to the exact information loss of edge connections. By the way, dropping edges randomly is able to create different random deformations of the input graph. In this way, DropEdge is able to prevent over-fitting, similar to typical image augmentation skills (\emph{e.g.} rotation, cropping and flapping) to hinder over-fitting in training CNNs. We will provide experimental validations in \textsection~\ref{sec:glasso}.
}


\textbf{Layer-Wise DropEdge.}
The above formulation of DropEdge is one-shot with all layers sharing the same perturbed adjacency matrix. Indeed, we can perform DropEdge for each individual layer. Specifically, we obtain $\hat{\mA}_{\text{drop}}^{(l)}$ by independently computing Eq.~\ref{Eq:DropEdge} for each $l$-th layer. Different layer could have different adjacency matrix $\hat{\mA}_{\text{drop}}^{(l)}$. Such layer-wise version brings in more randomness and deformation of the original data, and we will experimentally compare its performance with the original DropEdge in \textsection~\ref{sec:exp-dropedge}.

\subsection{Discussions}\label{sec.discussions}
This sections contrasts the difference between DropEdge and other related concepts including Dropout, DropNode, and Graph Sparsification. We also discuss the difference of over-smoothing between node classification and graph classification.

\textbf{DropEdge vs. Dropout.}
The Dropout trick~\cite{Hinton2012} is trying to perturb the feature matrix by randomly setting feature dimensions to be zeros, which may reduce the effect of over-fitting but is of no help to preventing over-smoothing since it does not make any change of the adjacency matrix. As a reference, DropEdge can be regarded as a generation of Dropout from dropping feature dimensions to dropping edges, which mitigates both over-fitting and over-smoothing. In fact, the impacts of Dropout and DropEdge are complementary to each other, and their compatibility will be shown in the experiments.

\textbf{DropEdge vs. DropNode.}
Another related vein belongs to the kind of node sampling based methods, including GraphSAGE~\cite{hamilton2017inductive}, FastGCN~\cite{chen2018fastgcn}, and AS-GCN~\cite{Huang2018}. We name this category of approaches as DropNode. For its original motivation, DropNode samples sub-graphs for mini-batch training, and it can also be treated as a specific form of dropping edges since the edges connected to the dropping nodes are also removed. However, the effect of DropNode on dropping edges is node-oriented and indirect. By contrast, DropEdge is edge-oriented, and it is possible to preserve all node features for the training (if they can be fitted into the memory at once), exhibiting more flexibility. Further, to maintain desired performance, the sampling strategies in current DropNode methods are usually inefficient, for example, GraphSAGE suffering from the exponentially-growing layer size (the number of sampled node), and AS-GCN requiring the sampling to be conducted recursively layer by layer. Our DropEdge, however, neither increases the layer size as the depth grows nor demands the recursive progress because the sampling of all edges are parallel.

\textbf{DropEdge vs. Graph-Sparsification.}
Graph-Sparsification~\cite{eppstein1997sparsification} is an old research topic in the graph domain. Its goal is removing unnecessary edges for graph compressing while keeping almost all information of the input graph. This is clearly district from the purpose of DropEdge where no optimization objective is needed. Specifically, DropEdge will remove the edges of the input graph by random at each training time, whereas Graph-Sparsification resorts to a tedious optimization method  to determine which edges to be deleted, and once those edges are discarded the output graph keeps unchanged.

\textbf{Node Classification vs. Graph Classification.}
The main focus of our paper is on node classification, where all nodes are in an identical graph. In graph classification, the nodes are distributed across different graphs; in this scenario, \autoref{th:universal} is still applicable per graph, and the node activations of an infinitely-deep GCN in the same graph instance are only distinguishable by node degrees. Yet, this is not true for those nodes in different graphs, since they will converge to different positions in $\gM$ (\emph{i.e.} $\mC$ in Definition~\ref{de:subspace}). To illustrate this, we suppose all graph instances are fully connected graphs and share the same form of $\hat{\mA}=\{\frac{1}{N}\}_{N\times N}$, the node features $\mX_i$ ($\geq0$) within graph $i$ are the same but different from those in different graphs, and the weight matrix is fixed as $\mW=\mI$ in Eq.~\ref{Eq:gcn}. Then, any layer of generic GCN keeps outputting $\mX_i$ for graph $i$, which indicates no information confusion happens across different graphs. 
Note that for graph classification over-smoothing per graph still hinders the expressive capability of GCN, as it will cause dimension collapse of the input data. 

\section{Experiments}
\label{sec:exps}
Our experimental evaluations are conducted with the goal to answer the following questions:
\begin{itemize}
    \item Is our proposed universal over-smoothing theorem in line with the experimental observation?
    \item How does our DropEdge help in relieving over-smoothing of different GCNs?
\end{itemize}
To address the first question, we display on a simulated dataset how the node activations will behave when the depth grows. We also calculate the distance between the node activations and the subspace to show the convergence dynamics. As for the second question, we contrast the classification performance of varying models of different depths with and without DropEdge on several real node classification benchmarks. The comparisons with state-of-the-art methods are involved as well. 

\begin{figure*}
\centering
\subfloat[Initial Visualization]{
\includegraphics[width=0.24\textwidth]{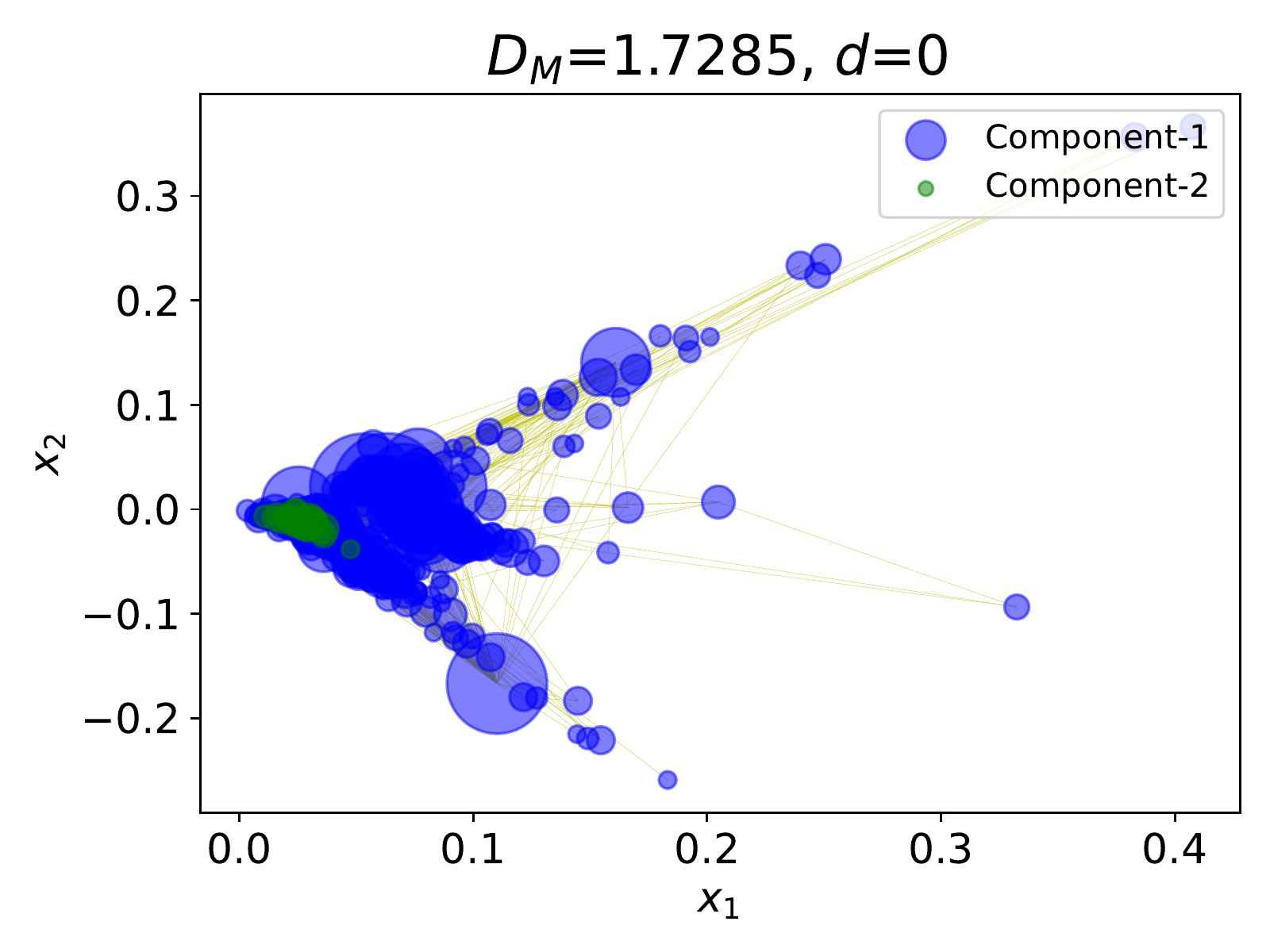}}
\subfloat[GCN-b]{
\includegraphics[width=0.24\textwidth]{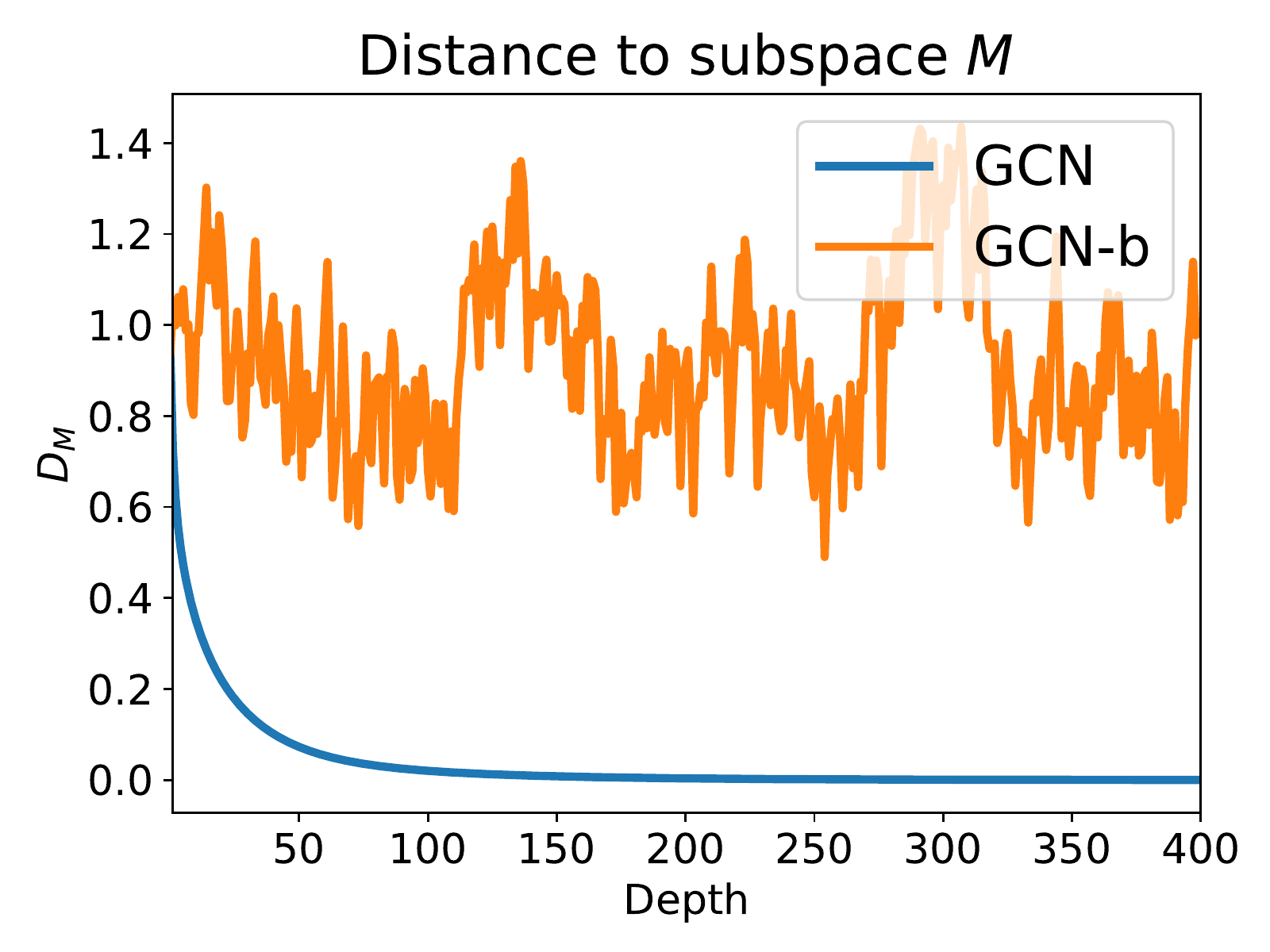}}
\subfloat[ResGCN]{
\includegraphics[width=0.24\textwidth]{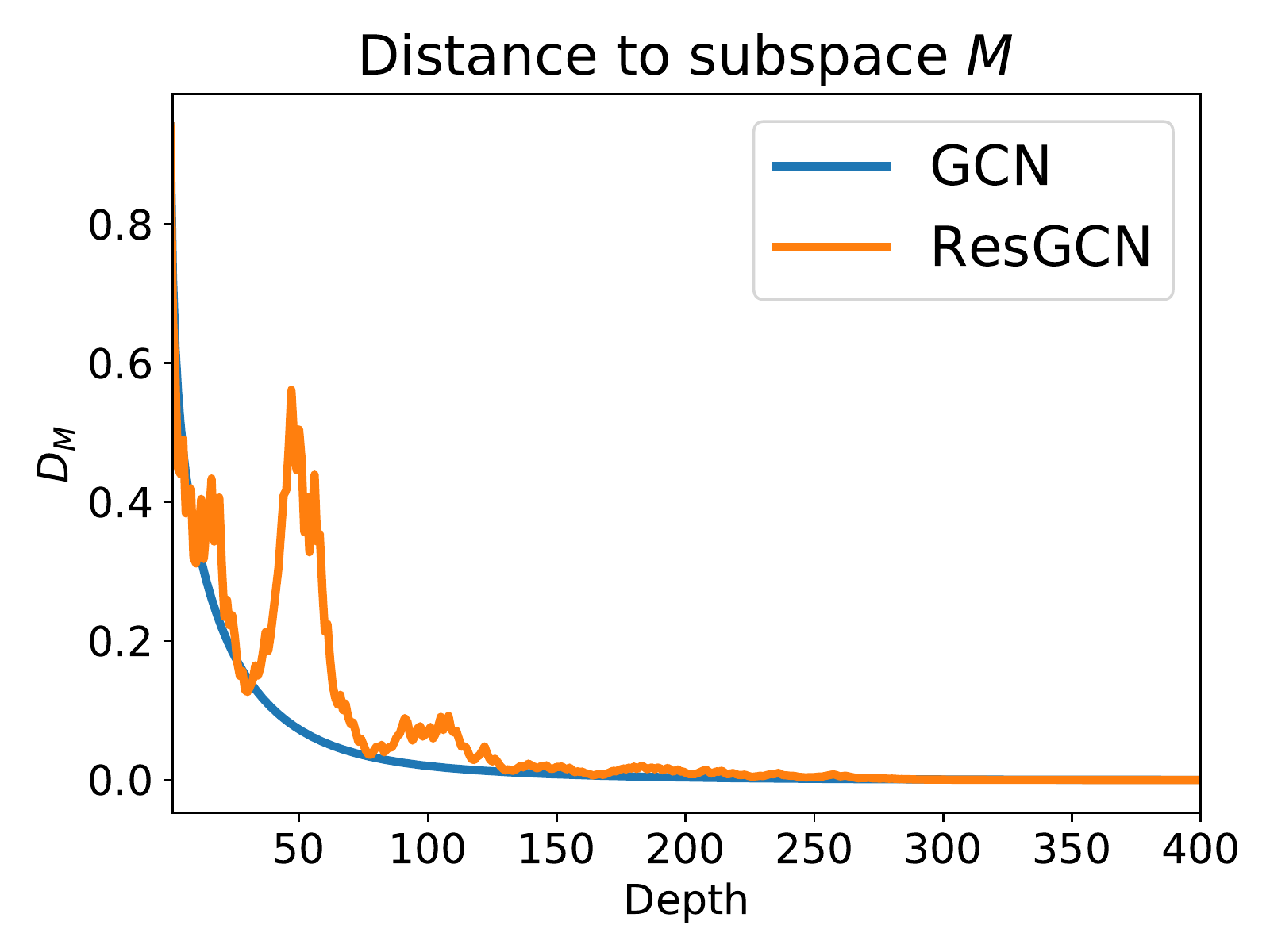}}
\subfloat[APPNP]{
\includegraphics[width=0.24\textwidth]{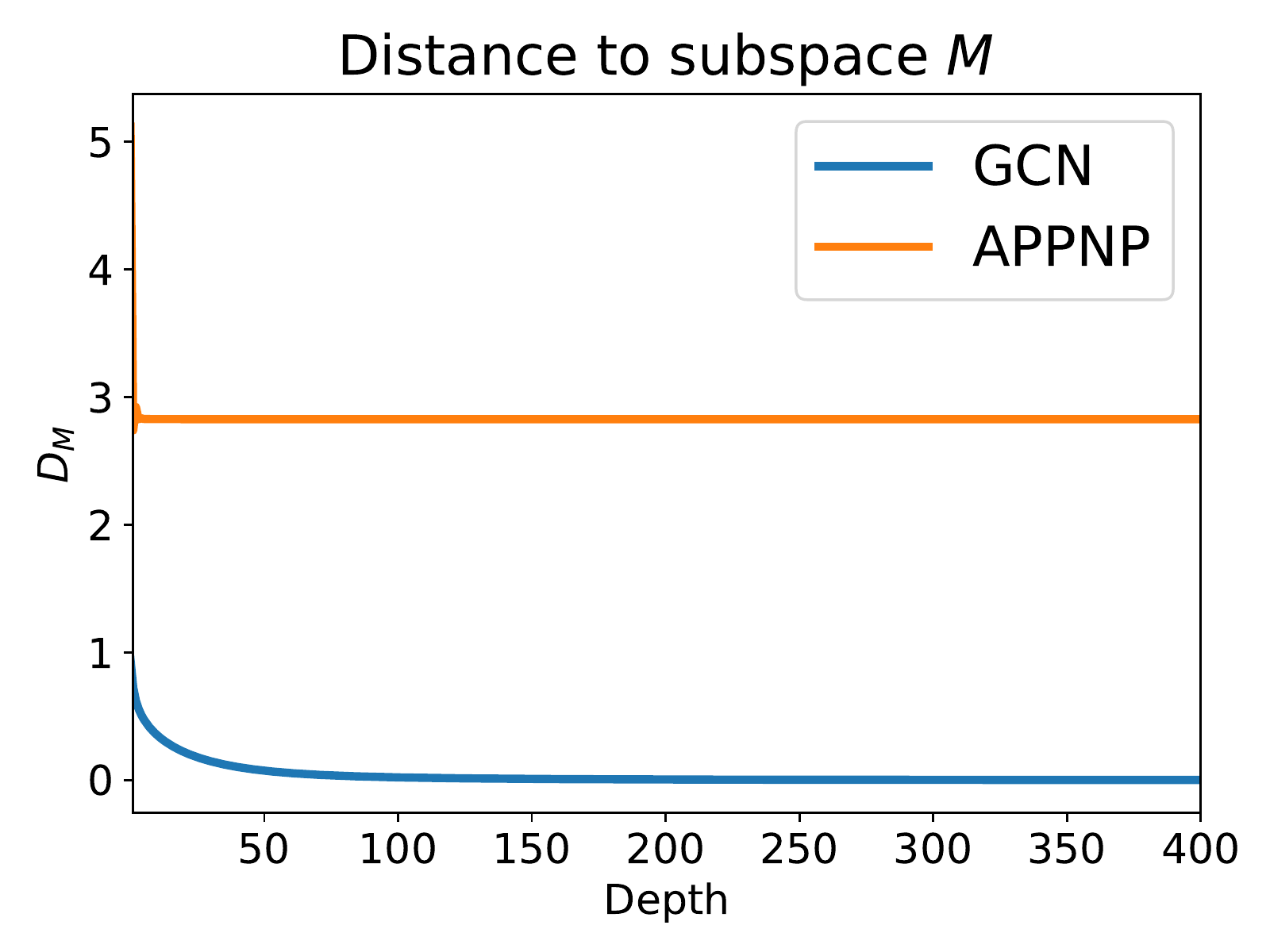}}\\
\subfloat[GCN]{
\includegraphics[width=0.24\textwidth]{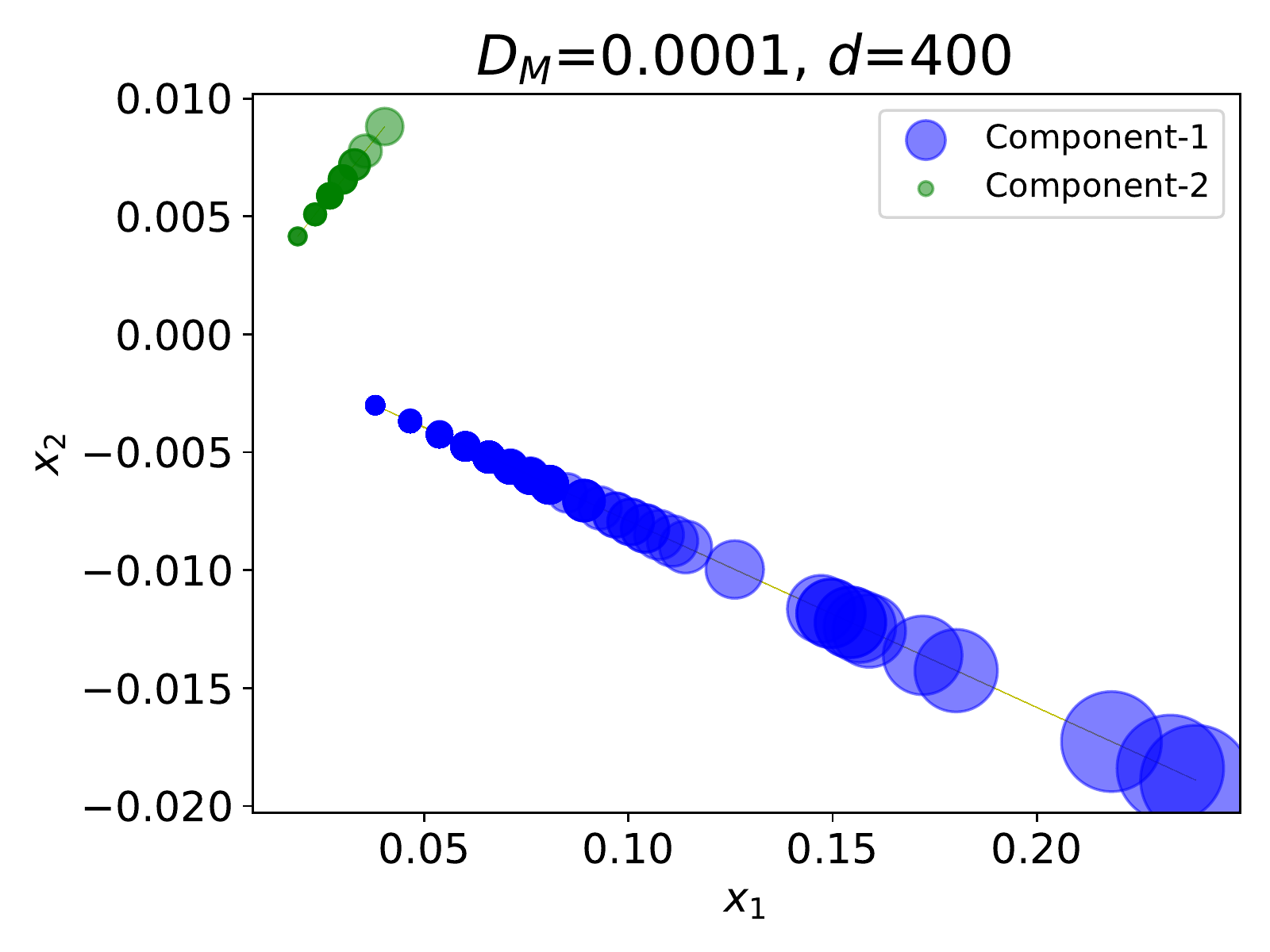}}
\subfloat[GCN-b]{
\includegraphics[width=0.24\textwidth]{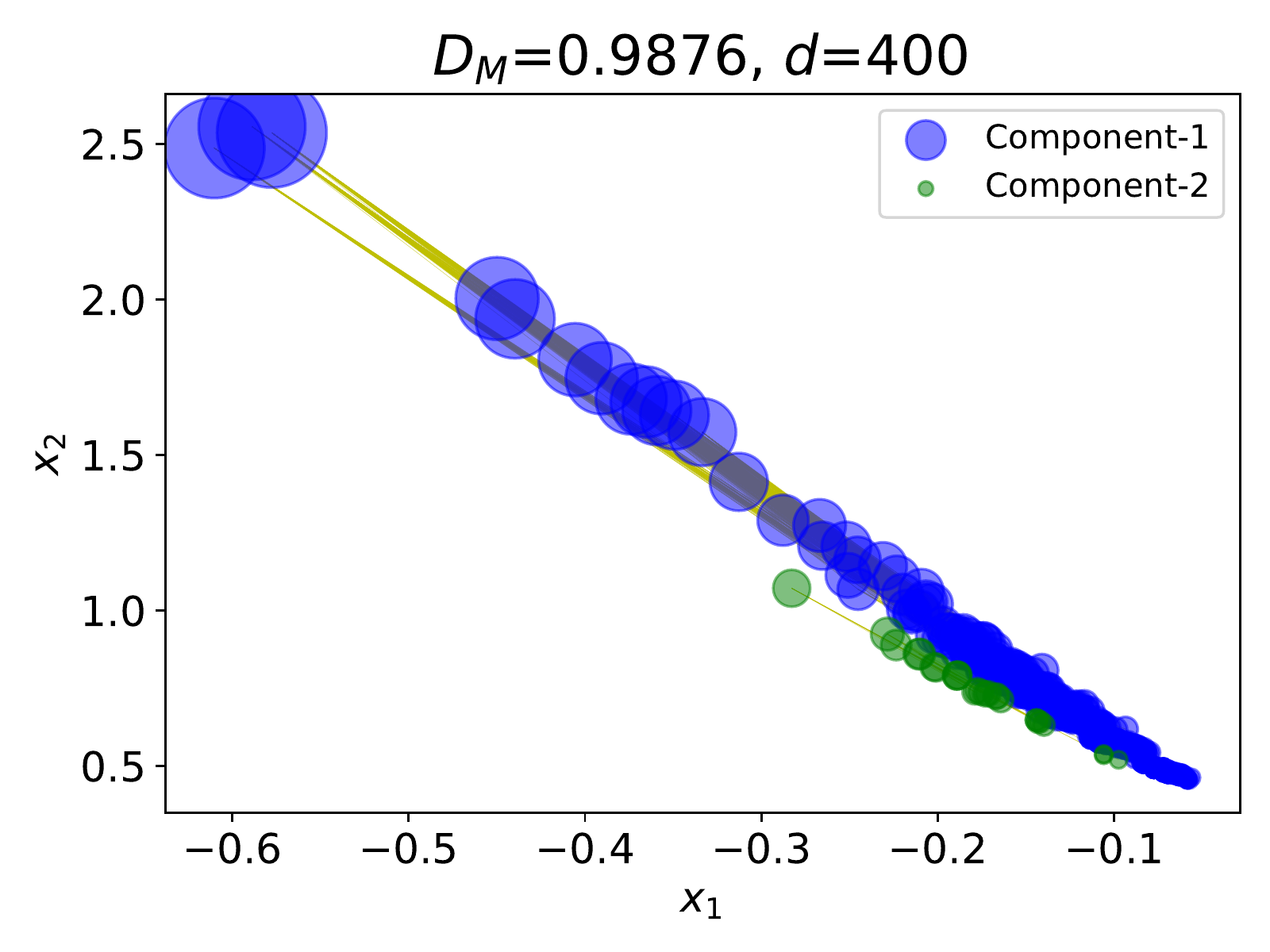}}
\subfloat[ResGCN]{
\includegraphics[width=0.24\textwidth]{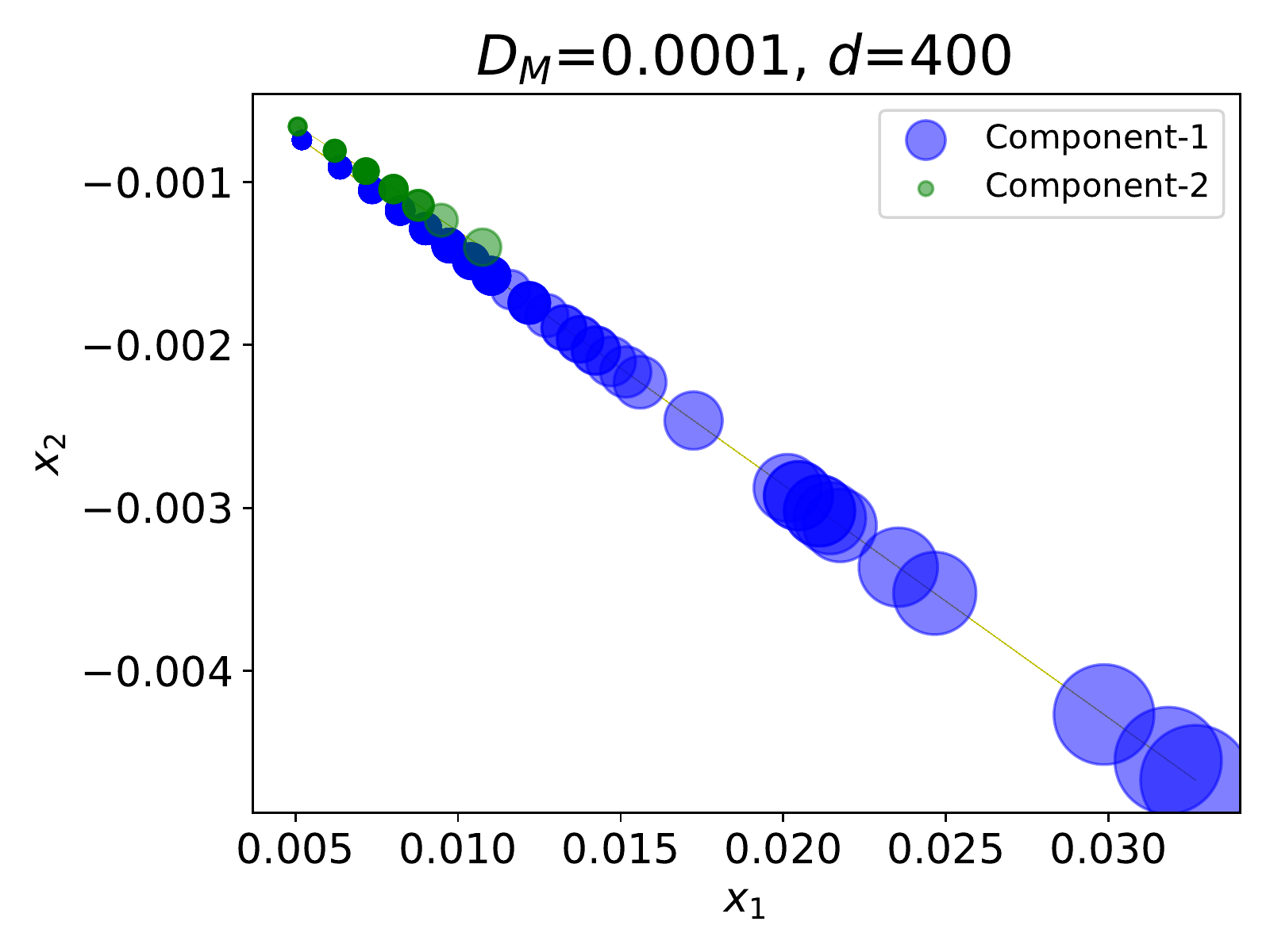}}
\subfloat[APPNP]{
\includegraphics[width=0.24\textwidth]{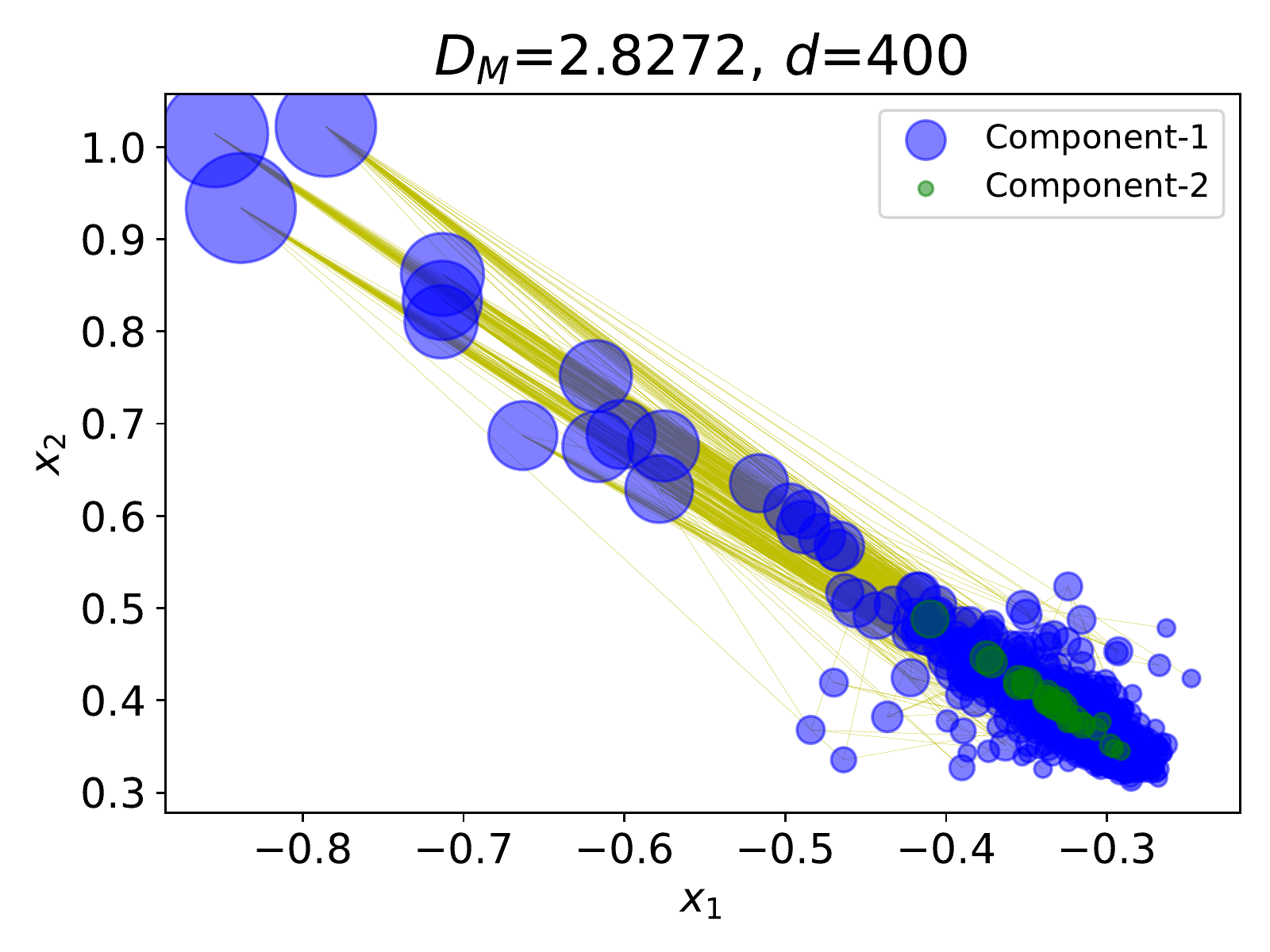}}

\caption{Output dynamics of GCN. (a) The initial visualization of the node distribution on Small-Cora, where the displayed size of each node reflects its degree. (b-d) The comparisons of the distance to the subspace $d_{\gM}$ between GCN and GCN-b, ResGCN and  APPNP, respectively, when the depth $d$ ranges from 0 to 400; (e-h) The visualization of the output for GCN, GCN-b, ResGCN and APPNP, when $d=400$.}
\label{fig.osm-gcn}
\end{figure*}

\begin{figure*}
\centering
\includegraphics[width=0.24\textwidth]{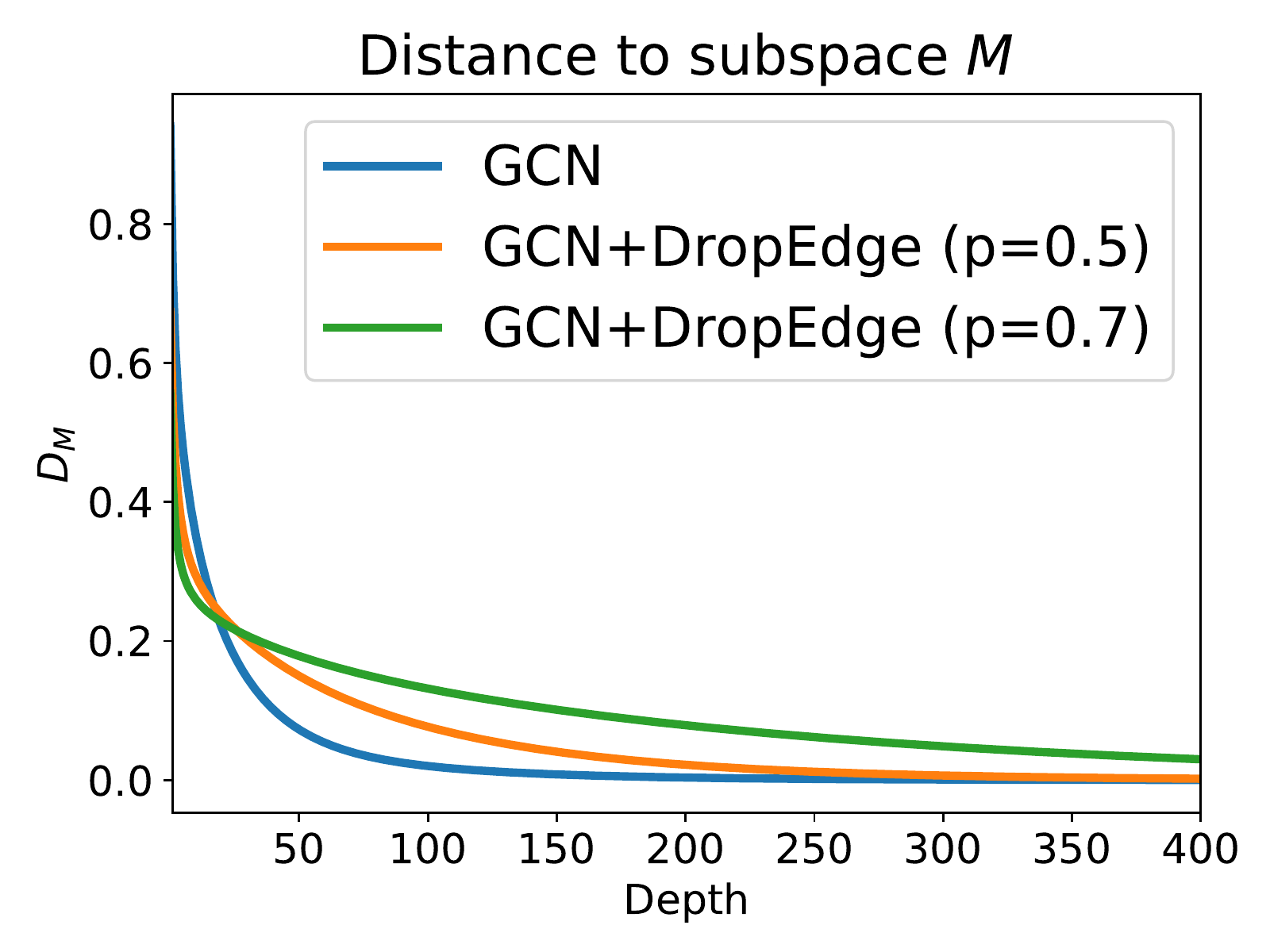}
\includegraphics[width=0.24\textwidth]{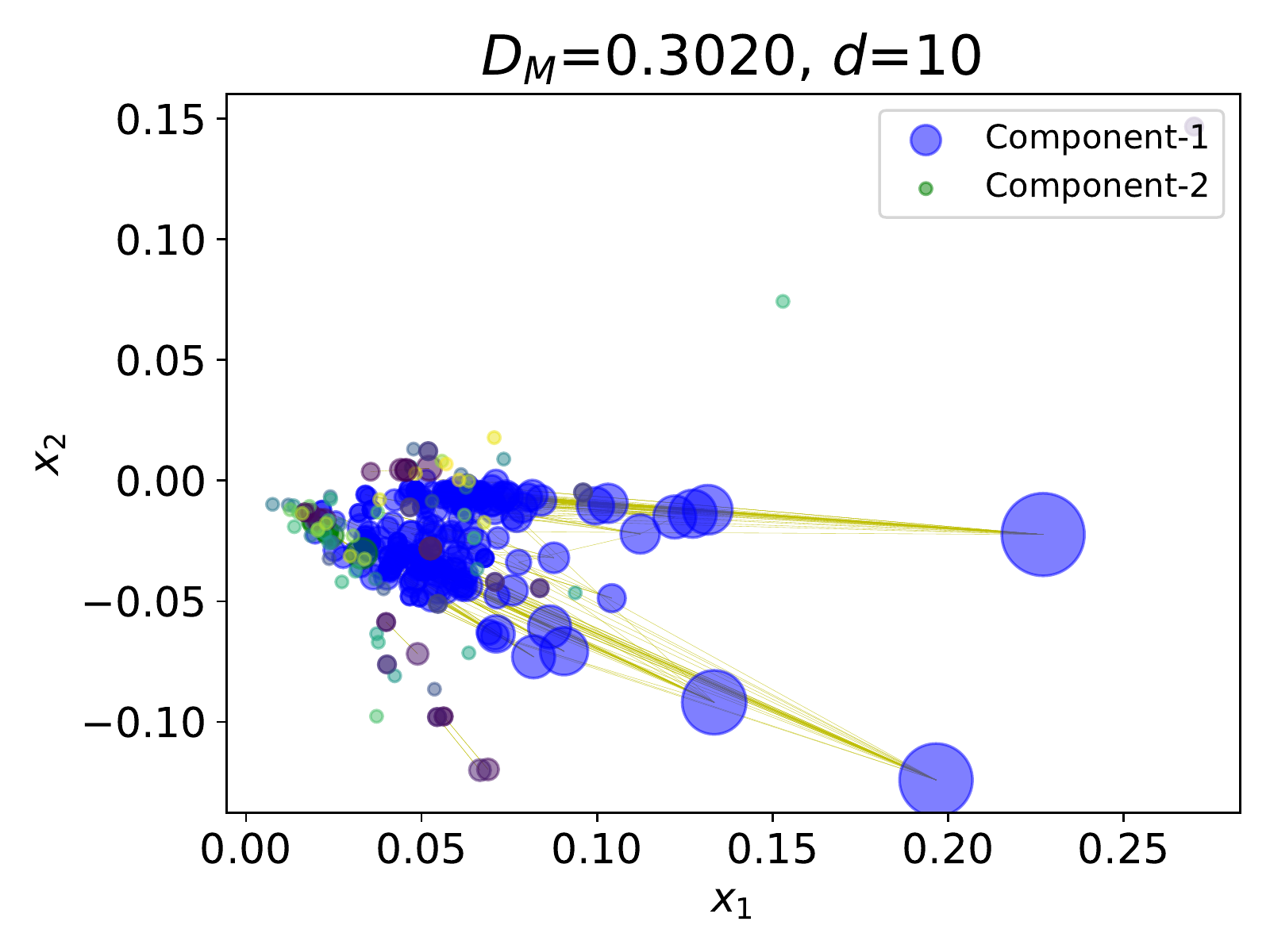}
\includegraphics[width=0.24\textwidth]{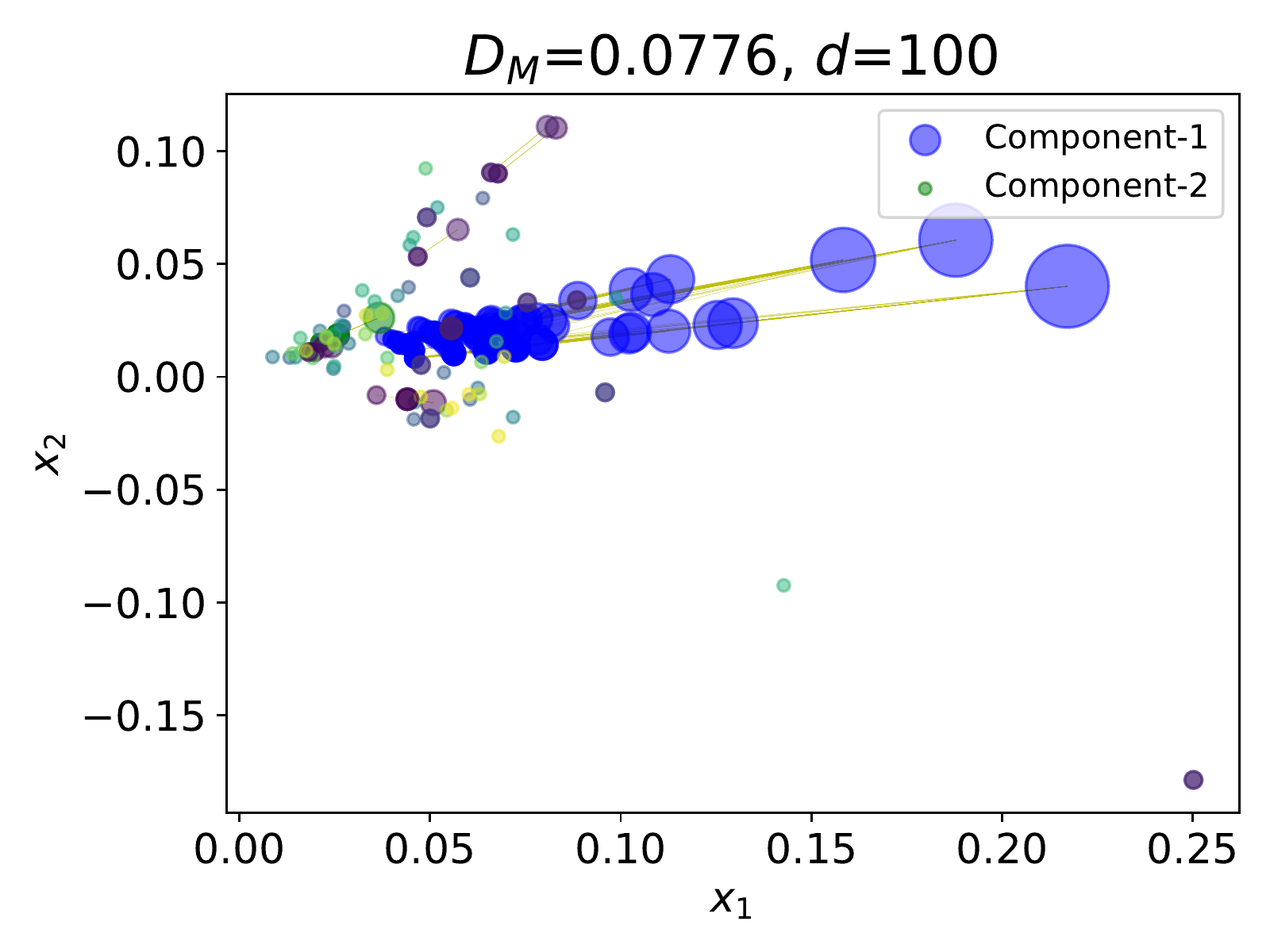}
\includegraphics[width=0.24\textwidth]{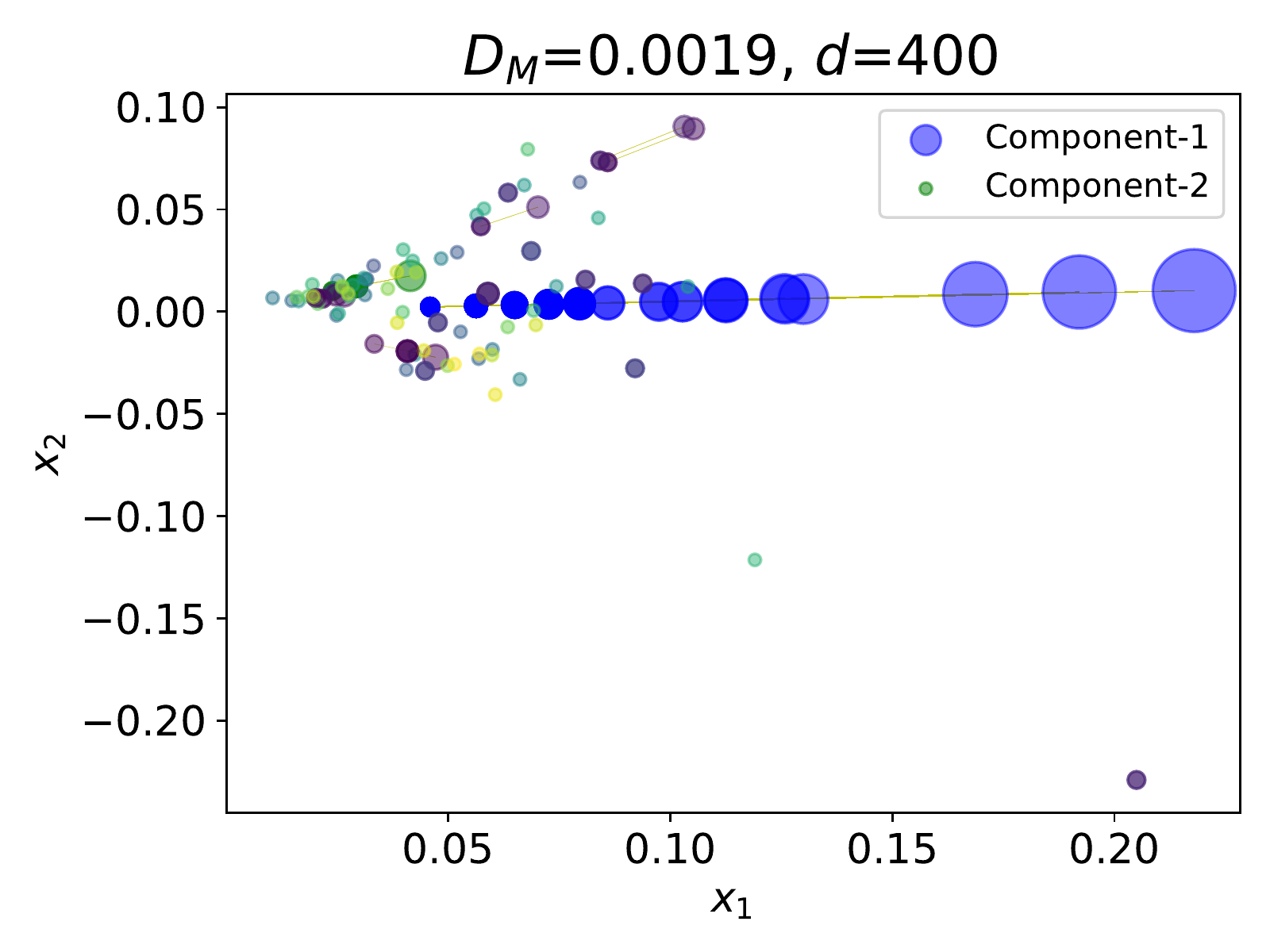}
\caption{Output dynamics of GCN with DropEdge. The left sub-figure plots $d_{\gM}$ for GCN and GCN with DropEdge (dropping rate $p=0.5, 0.7$) under varying depth. Other sub-figures depict the output when the depth is 10, 100, and 400 ($p=0.5$).}
\label{fig.osm-dropedge}
\end{figure*}

\textbf{Node classification datasets.} 
Joining the previous works' practice, we focus on four benchmark datasets varying in graph size and feature type: (1) classifying the research topic of papers in three citation datasets: Cora, Citeseer and Pubmed~\cite{sen2008collective}; (2) predicting which community different posts belong to in the Reddit social network~\cite{hamilton2017inductive}. Note that the tasks in Cora, Citeseer and Pubmed are transductive underlying all node features are accessible during training, while the task in Reddit is inductive meaning the testing nodes are unseen during training. We apply the full-supervised training fashion used in~\cite{Huang2018} and \cite{chen2018fastgcn} on all datasets in our experiments. The statics of all datasets are listed in Tab.~\ref{table:data}.

\textbf{Small-Cora.}
We have constructed a small dataset from Cora. In detail, we sample two connected components from the training graph of Cora, with the numbers of nodes being 654 and 26, respectively. The original feature dimension of nodes is 1433 which is not suitable for visualization on a 2-dimension plane. Hence, we apply truncated SVD for dimensionality reduction with output dimension as 2. Fig.~\ref{fig.osm-gcn} (a) displays the distribution of the node features. We call this simulated dataset as Small-Cora.

\subsection{Visualizing over-smoothing on Small-Cora}
\label{sec:small-cora}
Theorem~\ref{th:universal} has derived the universality of over-smoothing for the four models: GCN, GCN-b, ResGCN, and APPNP. Here, to check if it is consistent with empirical observations, we visualize the dynamics of the node activations on Small-Cora.  

\textbf{Implementations.}
To better focus on how the different structure of different GCN influences over-smoothing, the experiments in this section fix the hidden dimension of all GCNs to be 2, randomly initialize an orthogonal weight matrix $\mW$ for each layer and keep them untrained, which leads to $s=1$ in Eq.~\ref{equ:distance-b}. We also remove ReLU function, as we find that, with ReLU, the node activations will degenerate to zeros when the layer number grows, which will hinder the visualization. Regarding GCN-b, the bias of each layer is set as 0.05. We fix $\alpha=0.2$ and $\beta=0.5$ for ResGCN and APPNP, respectively. Since the total number of nodes is small (\emph{i.e.} 680), we are able to exactly devise the bases of the subspace, \emph{i.e.} $\mE$ according to Theorem~\ref{th:spectral}, and compute the distance between node activations $\mH$ and subspace $\gM$: $d_{\gM}$ using Eq. 11 in Appendix A. 
Fig.~\ref{fig.osm-gcn} demonstrates the output dynamics of all models. We have the following findings.

\textbf{For GCN}.
The nodes are generally distinguishable when $d=0$ (see (a)). After increasing $d$, the distance $d_{\gM}$ decreases dramatically, and it finally reaches very small value when $d=400$ (see (b) for example). It is clearly shown that, when $d=400$, the nodes within different components are collinear onto different lines, and the bigger (of larger degree) the node is, the farther it is from the zero point (see (e)). Such observation is consistent with Remark~\ref{re:gcn}, as different lines indeed represent different bases of the subspace.

\textbf{For GCN-b}.
The output dynamics of GCN-b is distinct from GCN. It turns out the nodes within the same component keep non-collinear when $d=400$, as shown in (f). In (b), in contrast to GCN, the curve of GCN-b fluctuates within a certain bounded area. This result coincides with Remark~\ref{re:gcn-b} and supports that adding the bias term enables the node activations to converge to a cuboid surrounding the subspace under a certain radium.

\textbf{For ResGCN}.
Akin to GCN, the output of ResGCN approaches the subspace in the end, but its convergence dynamics as shown in (c) is a bit different. The curve shakes up and down for several rounds before it eventually degenerates. This could because the skip connection in ResGCN helps prevent over-smoothing or even reverse the convergence direction during the early period. When $d=400$ (in (g)),    
each node will exponentially fall into the subspace at the rate of $\lambda+\alpha$ as proven in Remark~\ref{re:resgcn}. Note that the average speed of ResGCN is smaller than that of GCN (recalling $\lambda+\alpha>\lambda$).

\textbf{For APPNP}.
The behavior of APPNP is completely different from GCN in (d). It quickly becomes stationary and this stationary point is beyond the subspace up to a fixed distance $r>0$, which confirms the conclusion by Remark~\ref{re:appnp}. In APPNP, as the rate $v=\lambda\beta$ is smaller than that of GCN, its convergence speed is faster.  

In addition, Fig.~\ref{fig.osm-dropedge} demonstrates how DropEdge changes the dynamics of GCN.  
Clearly, the results verify Theorem~\ref{The:smoothing}-\ref{The:smoothing-2}, where the convergence to the subspace becomes slower and the number of connected components is larger when we perform DropEdge on GCN with the dropping rate $p=0.5$. If we further increase $p$ to 0.7, the convergence speed will be further decreased.

\begin{table*}[htbp]
  \centering
   \renewcommand\arraystretch{1.1}
  \caption{Testing accuracy (\%) on different backbones. }
  \setlength{\tabcolsep}{9pt}
\begin{tabular}{|c|c|l|rrrrrr|r|}
\toprule
\multicolumn{3}{|c|}{Layer} & 2     & 4     & 8     & 16    & 32    & 64    & \multicolumn{1}{l|}{Average  Improvement}  \\
\midrule
\multirow{10}[20]{*}{Cora} & \multirow{2}[4]{*}{GCN} & Original & 85.8  & 85.6  & 83.2  & 81.2  & 69.8  & 42.1  & -  \\
\cline{3-3}      &       & DropEdge & \textbf{86.4} & \textbf{86.6} & \textbf{85.5} & \textbf{84.2} & \textbf{71.3} & \textbf{50.6} & \textbf{+2.8}  \\
\cline{2-10}      & \multirow{2}[4]{*}{GCN-b} & Original & 86.1  & 85.5  & 78.7  & 82.1  & 71.6  & 52.0  & -  \\
\cline{3-3}      &       & DropEdge & \textbf{86.5} & \textbf{87.6} & \textbf{85.8} & \textbf{84.3} & \textbf{74.6} & \textbf{53.2} & \textbf{+2.7}  \\
\cline{2-10}      & \multirow{2}[4]{*}{ResGCN} & Original & -     & 86.0  & 85.4  & 85.3  & 85.1  & 79.8  & -  \\
\cline{3-3}      &       & DropEdge & \textbf{-} & \textbf{87.0} & \textbf{86.9} & \textbf{86.9} & \textbf{86.8} & \textbf{84.8} & \textbf{+2.2}  \\
\cline{2-10}      & \multirow{2}[4]{*}{JKNet} & Original & -     & 86.9  & 86.7  & 86.2  & 87.1  & 86.3  & \textbf{-}  \\
\cline{3-3}      &       & DropEdge & \textbf{-} & \textbf{87.7} & \textbf{87.8} & \textbf{88.0} & \textbf{87.6} & \textbf{87.9} & \textbf{+1.2}  \\
\cline{2-10}      & \multirow{2}[4]{*}{APPNP} & Original & -     & 87.9  & 87.7  & 87.5  & 87.8  & 87.5  & -  \\
\cline{3-3}      &       & DropEdge & \textbf{-} & \textbf{88.6} & \textbf{89.0} & \textbf{88.8} & \textbf{88.9} & \textbf{89.1} & \textbf{+1.2} \\
\hline
\multirow{10}[20]{*}{Citeseer} & \multirow{2}[4]{*}{GCN} & Original & 76.8  & 72.7  & 72.6  & 72.2  & 56.5  & 43.8  & - \\
\cline{3-3}      &       & DropEdge & \textbf{78.1} & \textbf{79.0} & \textbf{75.4} & \textbf{67.5} & \textbf{60.5} & \textbf{46.4} & \textbf{+2.1} \\
\cline{2-10}      & \multirow{2}[4]{*}{GCN-b} & Original & 75.9  & 76.7  & 74.6  & 65.2  & 59.2  & 44.6  & - \\
\cline{3-3}      &       & DropEdge & \textbf{78.7} & \textbf{79.2} & \textbf{77.2} & \textbf{76.8} & \textbf{61.4} & \textbf{45.6} & \textbf{+3.8} \\
\cline{2-10}      & \multirow{2}[4]{*}{ResGCN} & Original & -     & 78.9  & 77.8  & 78.2  & 74.4  & 21.2  & - \\
\cline{3-3}      &       & DropEdge & \textbf{-} & \textbf{78.8} & \textbf{78.8} & \textbf{79.4} & \textbf{77.9} & \textbf{75.3} & \textbf{+11.9} \\
\cline{2-10}      & \multirow{2}[4]{*}{JKNet} & Original & -     & 79.1  & 79.2  & 78.8  & 71.7  & 76.7  & - \\
\cline{3-3}      &       & DropEdge & \textbf{-} & \textbf{80.2} & \textbf{80.2} & \textbf{80.1} & \textbf{80.0} & \textbf{80.0} & \textbf{+3.0} \\
\cline{2-10}      & \multirow{2}[4]{*}{APPNP} & Original & -     & 80.3  & 80.5  & 80.2  & 79.9  & 80.4  & - \\
\cline{3-3}      &       & DropEdge & \textbf{-} & \textbf{80.8} & \textbf{80.9} & \textbf{81.1} & \textbf{81.2} & \textbf{81.3} & \textbf{+0.8} \\
\hline
\multirow{10}[20]{*}{Pubmed} & \multirow{2}[4]{*}{GCN} & Original & 89.5  & 89.2  & 88.3  & 87.7  & 78.6  & 72.7  & - \\
\cline{3-3}      &       & DropEdge & \textbf{89.7} & \textbf{90.9} & \textbf{91.0} & \textbf{90.5} & \textbf{80.1} & \textbf{77.5} & \textbf{+2.3} \\
\cline{2-10}      & \multirow{2}[4]{*}{GCN-b} & Original & 90.2  & 88.7  & 90.1  & 88.1  & 84.6  & \textbf{79.7} & - \\
\cline{3-3}      &       & DropEdge & \textbf{91.2} & \textbf{91.3} & \textbf{90.9} & \textbf{90.3} & \textbf{86.2} & 79.0  & \textbf{+1.2} \\
\cline{2-10}      & \multirow{2}[4]{*}{ResGCN} & Original & -     & 90.7  & 89.6  & 89.6  & 90.2  & 87.9  & - \\
\cline{3-3}      &       & DropEdge & \textbf{-} & \textbf{90.7} & \textbf{90.5} & \textbf{91.0} & \textbf{91.1} & \textbf{90.2} & \textbf{+1.1} \\
\cline{2-10}      & \multirow{2}[4]{*}{JKNet} & Original & -     & 90.5  & 90.6  & 89.9  & 89.2  & 90.6  & - \\
\cline{3-3}      &       & DropEdge & \textbf{-} & \textbf{91.3} & \textbf{91.2} & \textbf{91.5} & \textbf{91.3} & \textbf{91.6} & \textbf{+1.2} \\
\cline{2-10}      & \multirow{2}[4]{*}{APPNP} & Original & -     & 90.4  & 90.3  & 89.8  & 90.0  & 90.3  & - \\
\cline{3-3}      &       & DropEdge & \textbf{-} & \textbf{90.7} & \textbf{90.4} & \textbf{90.3} & \textbf{90.5} & \textbf{90.5} & \textbf{+0.3} \\
\hline
\multirow{10}[20]{*}{Reddit} & \multirow{2}[4]{*}{GCN} & Original & 95.75 & 96.08 & 96.43 & 79.87 & 44.36 & -     & - \\
\cline{3-3}      &       & DropEdge & \textbf{95.93} & \textbf{96.23} & \textbf{96.57} & \textbf{89.02} & \textbf{66.18} & - & \textbf{+6.3}  \\
\cline{2-10}      & \multirow{2}[4]{*}{GCN-b} & Original & 96.11 & 96.62 & 96.17 & 67.11 & 45.55 & -     & - \\
\cline{3-3}      &       & DropEdge & \textbf{96.13} & \textbf{96.71} & \textbf{96.48} & \textbf{90.54} & \textbf{50.51} & - & \textbf{+5.8}  \\
\cline{2-10}      & \multirow{2}[4]{*}{ResGCN} & Original & -     & 96.13 & 96.37 & 96.34 & 93.93 & -     & - \\
\cline{3-3}      &       & DropEdge & \textbf{-} & \textbf{96.33} & \textbf{96.46} & \textbf{96.48} & \textbf{94.27} & \textbf{-} & \textbf{+0.2} \\
\cline{2-10}      & \multirow{2}[4]{*}{JKNet} & Original & -     & 96.54 & 96.82 & 95.91 & 95.42 & -     & - \\
\cline{3-3}      &       & DropEdge & \textbf{-} & \textbf{96.75} & \textbf{97.02} & \textbf{96.39} & \textbf{95.68} & \textbf{-} & \textbf{+0.3} \\
\cline{2-10}      & \multirow{2}[4]{*}{APPNP} & Original & -     & 95.84 & 95.77 & 95.64 & 95.59 & -     & - \\
\cline{3-3}      &       & DropEdge & \textbf{-} & \textbf{95.89} & \textbf{95.91} & \textbf{95.76} & \textbf{95.73} & - & \textbf{+0.1} \\
\bottomrule
\end{tabular}%
  \label{tab:overall_res}%
\end{table*}

\subsection{Evaluating DropEdge on different GCNs}\label{sec.cmpdropedge}

In this section, we are interested in if applying DropEdge can promote the performance of the aforementioned four GCN models on real node classification benchmarks: Cora, Citeseer, Pubmed, and Reddit. We further implement JKNet~\cite{xu2018representation} and carry out DropEdge on it. 
We denote each model X of depth $d$ as X-$d$ for short in what follows, \emph{e.g.} GCN-4 denotes the 4-layer GCN.


\textbf{Implementations.}
Different from~\textsection~\ref{sec:small-cora}, the parameters of all models are trainable and initialized with the method proposed by~\cite{Kipf2017}, and the ReLU function is added. We implement all models on all datasets with the depth $d\in\{2,4,8,16,32,64\}$ and the hidden dimension $128$. For Reddit, the maximum depth is 32 considering the memory bottleneck. Since different structure exhibits different training dynamics on different dataset, to enable more robust comparisons, we perform random hyper-parameter search for each model, and report the case giving the best accuracy on validation set of each benchmark. The searching space of hyper-parameters and more details are provided in Tab.~7 in Appendix~D. Tab.~8 depicts different type of normalizing the adjacency matrix, and the selection of normalization is treated as a hyper-parameter. Regarding the same architecture w or w/o DropEdge, we apply the same set of hyper-parameters except the drop rate $p$ for fair evaluation. We adopt the Adam optimizer for model training. To ensure the re-productivity of the results, the seeds of the random numbers of all experiments are set to the same. We fix the number of training epoch to $400$ for all datasets. All experiments are conducted on a NVIDIA Tesla P40 GPU with 24GB memory. Tab.~9 in Appendix summaries the hyper-parameters of each backbone with the best accuracy on different datasets.

\textbf{Overall Results.}
Tab.~\ref{tab:overall_res} summaries the results on all datasets. We have the following observations:
\begin{itemize}
    \item DropEdge consistently improves the testing accuracy of the model without DropEdge for all cases. On Citeseer, for example, ResGCN-64 fails to produce meaningful classification performance while ResGCN-64 with DropEdge still delivers promising result.
    \item For deep models, GCN-b, ResGCN, and APPNP generally outperform generic GCN with or without DropEdge, which is consistent with our before theoretical analyses. In particular, APPNP+DropEdge performs best on Cora and Citeseer, while JKNet+DropEdge yields the best result on Pubmed and Reddit, showing the compatibility of DropEdge to modern architectures.
    \item After using DropEdge, all models normally achieves the highest accuracy when the depth is sufficiently large. For instance, both GCN and GCN-b perform worse when $d$ increases, but with DropEdge, they both arrive at the peak when $d=4$, probably thanks to the alleviation of over-soothing by DropEdge.
\end{itemize}
In addition, the validation losses of all 4-layer and 6-layer models on Cora and Citeseer are shown in Figure~\ref{fig.dropvallosscmpaddtional}. The curves of both training and validation are dramatically pulled down after applying DropEdge, which also explain the benefit of DropEdge.

\begin{table*}[t!]
  \centering
  \caption{Test Accuracy (\%) comparison with SOTAs. The number in parenthesis denotes the network
depth.}
  \vskip -0.1 in
 \renewcommand\arraystretch{0.9}
\setlength{\tabcolsep}{21pt}
    \begin{tabular}{c|c|c|c|c}
    \toprule
    & Cora & Citeseer & Pubmed & Reddit \\
    \midrule
  KLED~\cite{Fouss06anexperimental}  & $82.3$ & - & $82.3$    & -       \\
  DCNN~\cite{atwood2016diffusion}  & $86.8$ & - & $89.8$    & -       \\
  GAT~\cite{DBLP:journals/corr/abs-1710-10903}   & $87.4$ & $78.6$ & $89.7$ &- \\
  FastGCN~\cite{chen2018fastgcn} & $85.0$ & $77.6$ & $88.0$ & $93.7$ \\
  AS-GCN~\cite{Huang2018} & $87.4$ & $79.7$ & $90.6$ & $96.3$ \\
  GraphSAGE~\cite{hamilton2017inductive} & $82.2$ & $71.4$ & $87.1$ & $94.3$ \\
   DAGNN~\cite{liu2020towards} & $88.4$ & $78.6$ & $86.4$  & OOM \\
	\hline
     \hline
    GCN+DropEdge      & $86.6(4)$ 	  	 & $79.0(4)$   & $91.0(8)$        & $96.57(8)$ \\
    GCN-b+DropEdge    & $87.6(4)$   		 & $79.2(4)$   & $91.3(4)$    & $96.71(4)$ \\
    ResGCN+DropEdge   & $87.0(4)$   		 & $79.4(16)$   & $91.1(32)$  & $96.48(16)$\\
    APPNP+DropEdge    & $\textbf{89.1(64)}$ & $\textbf{81.3(64)}$         & $90.7(4)$  & $95.91(8)$\\
    JKNet+DropEdge    & $88.0(16)$ & $80.2(8)$  & $\textbf{91.6(64)}$    & $\textbf{97.02(8)}$\\
    DAGNN+DropEdge &$88.7(8)$ &$79.5(8)$ & $87.1(16)$ & OOM \\
    \bottomrule
    \end{tabular}%
  \label{tab:full_sota}%
\end{table*}

\textbf{Comparison with SOTAs}
We select the best performance for each backbone with DropEdge, and contrast them with existing State of the Arts (SOTA), including KLED~\cite{Fouss06anexperimental}, DCNN~\cite{atwood2016diffusion}, FastGCN~\cite{chen2018fastgcn}, AS-GCN~\cite{Huang2018}, , GraphSAGE~\cite{Hamilton2017} and DAGNN~\cite{liu2020towards} in Tab.~\ref{tab:full_sota}; for the SOTA methods, we reuse the results reported in~\cite{Huang2018}. 
We have these findings in Tab.~\ref{tab:full_sota}: 
\begin{itemize}
    \item Clearly, our DropEdge obtains significant enhancement against SOTAs; particularly on Cora and Citeseer, the best accuracies by APPNP+DropEdge are 89.10\% and 81.30\%, which are clearly better than the previous best (87.44\% and 79.7\%), and obtain around 1\% improvement compared to the no-drop APPNP. Such improvement is regarded as a remarkable boost considering the challenge on these benchmarks.
    \item For most models with DropEdge, the best accuracy is obtained under the depth beyond 4, which again verifies the impact of DropEdge on formulating deep networks.
    \item As mentioned in \textsection~\ref{sec.discussions}, FastGCN, AS-GCN and GraphSAGE are considered as the DropNode extensions of GCNs. The DropEdge based approaches outperform the DropNode based variants as shown in Tab.~\ref{tab:full_sota}, which confirms the effectiveness of DropEdge. 
    \item DAGNN is a recently proposed model that is able to alleviate over-smoothing. Table~\ref{tab:full_sota} also reports the performance of DAGNN, where we have conducted DAGNN with varying depth and collect the best case on each dataset. It is observed that adding DropEdge on DAGNN can further boost the performance, implying the generality of our proposed method.
\end{itemize}

\begin{figure*}
\centering
\includegraphics [width=0.24\textwidth]{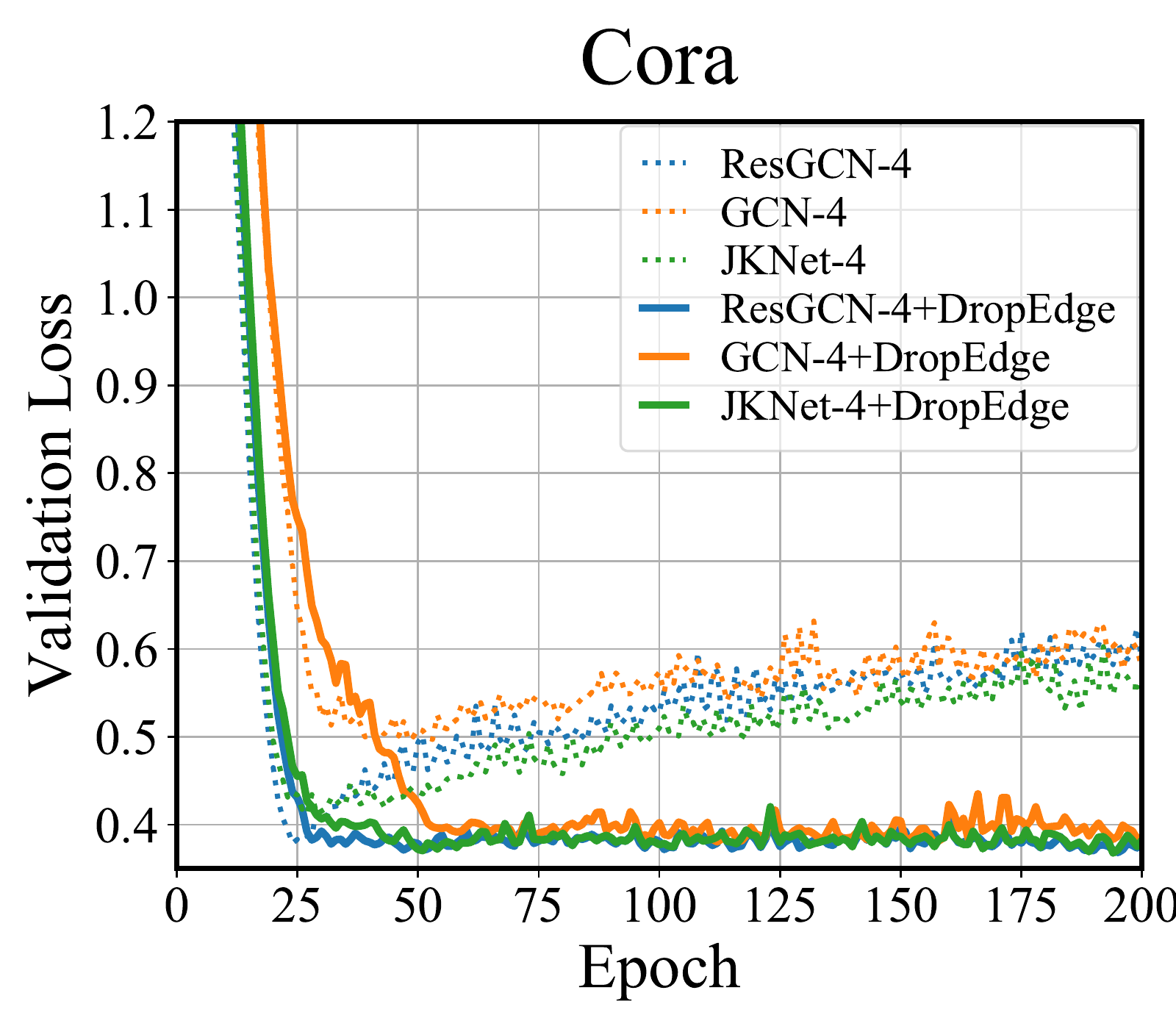}
\includegraphics [width=0.24\textwidth]{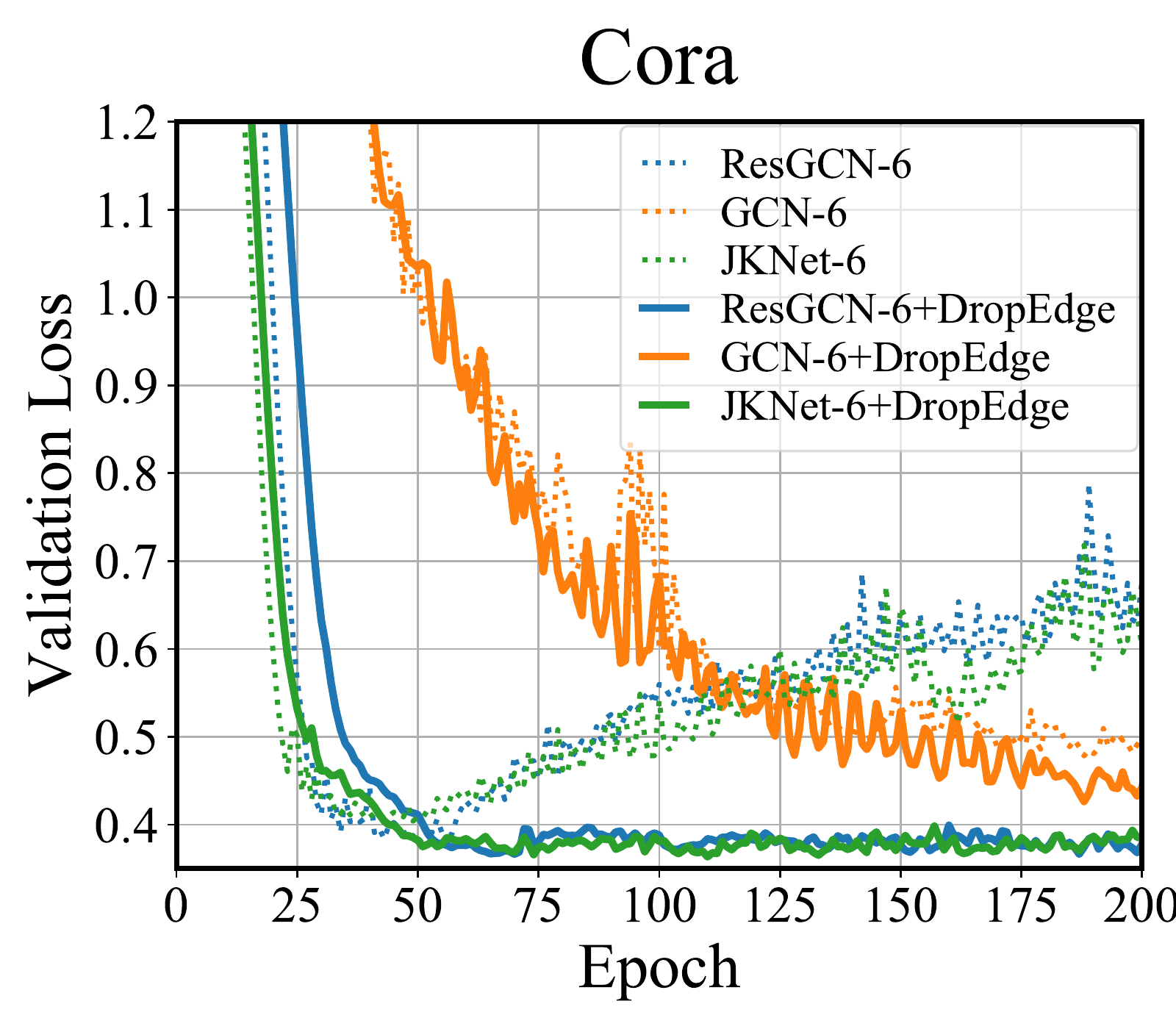}
\includegraphics [width=0.24\textwidth]{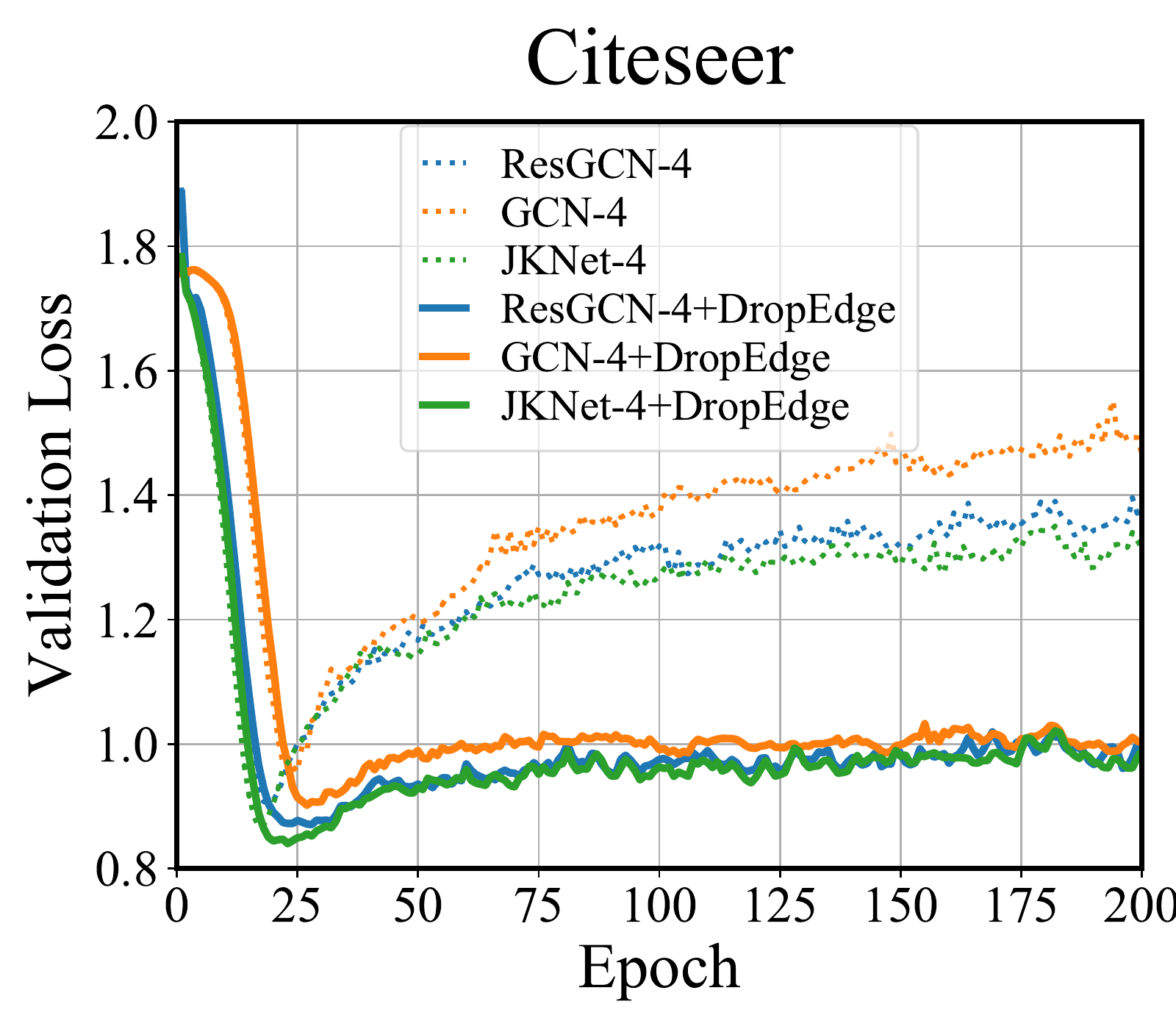}
\includegraphics [width=0.24\textwidth]{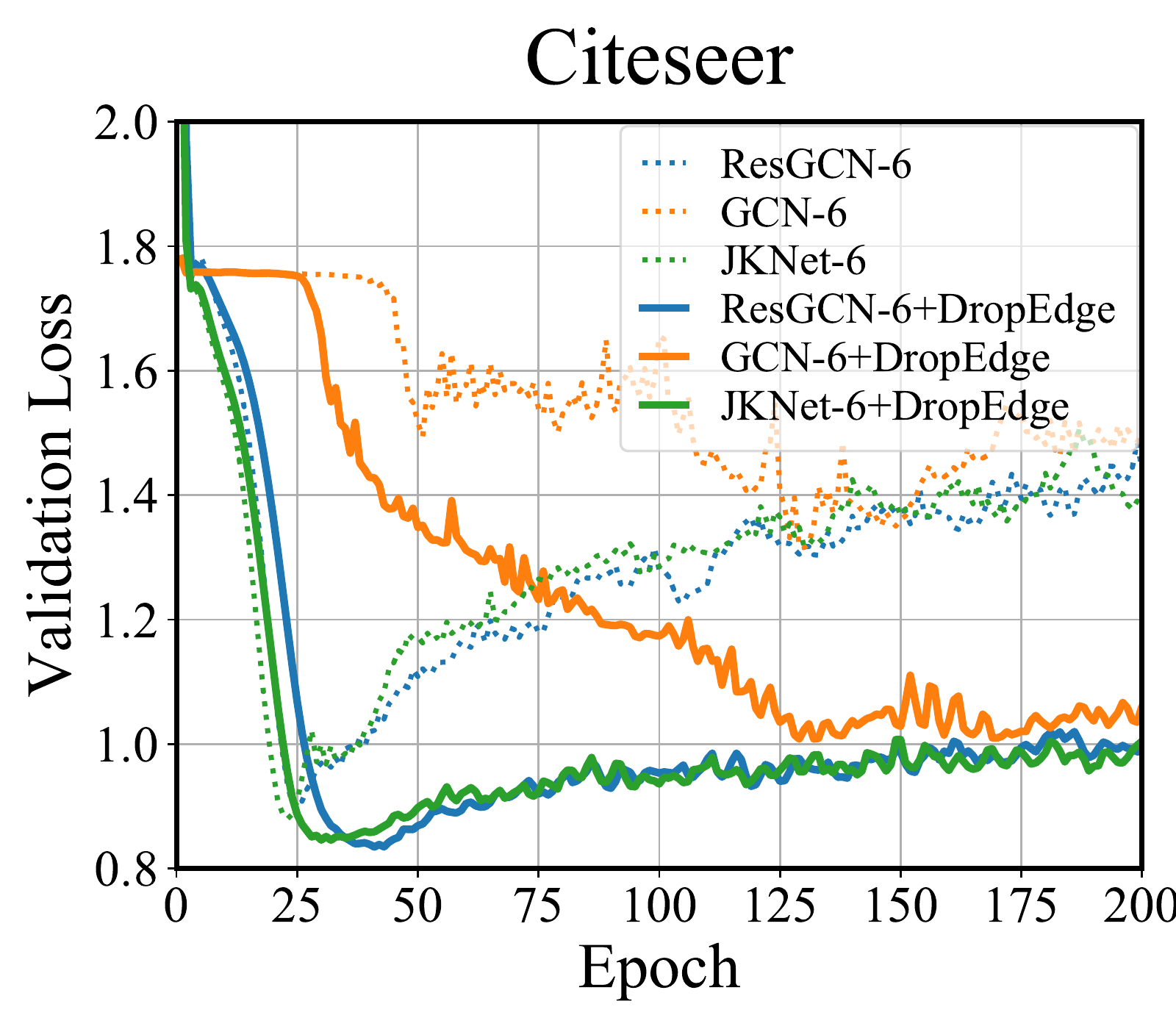}
\caption{The validation loss on different backbones w and w/o DropEdge. GCN-$n$ denotes GCN of depth $n$; similar denotation follows for other backbones.}
\label{fig.dropvallosscmpaddtional}
\end{figure*}

\begin{figure*}[t!]
    \centering
    \subfloat[]{    
    \includegraphics[width=0.24\textwidth]{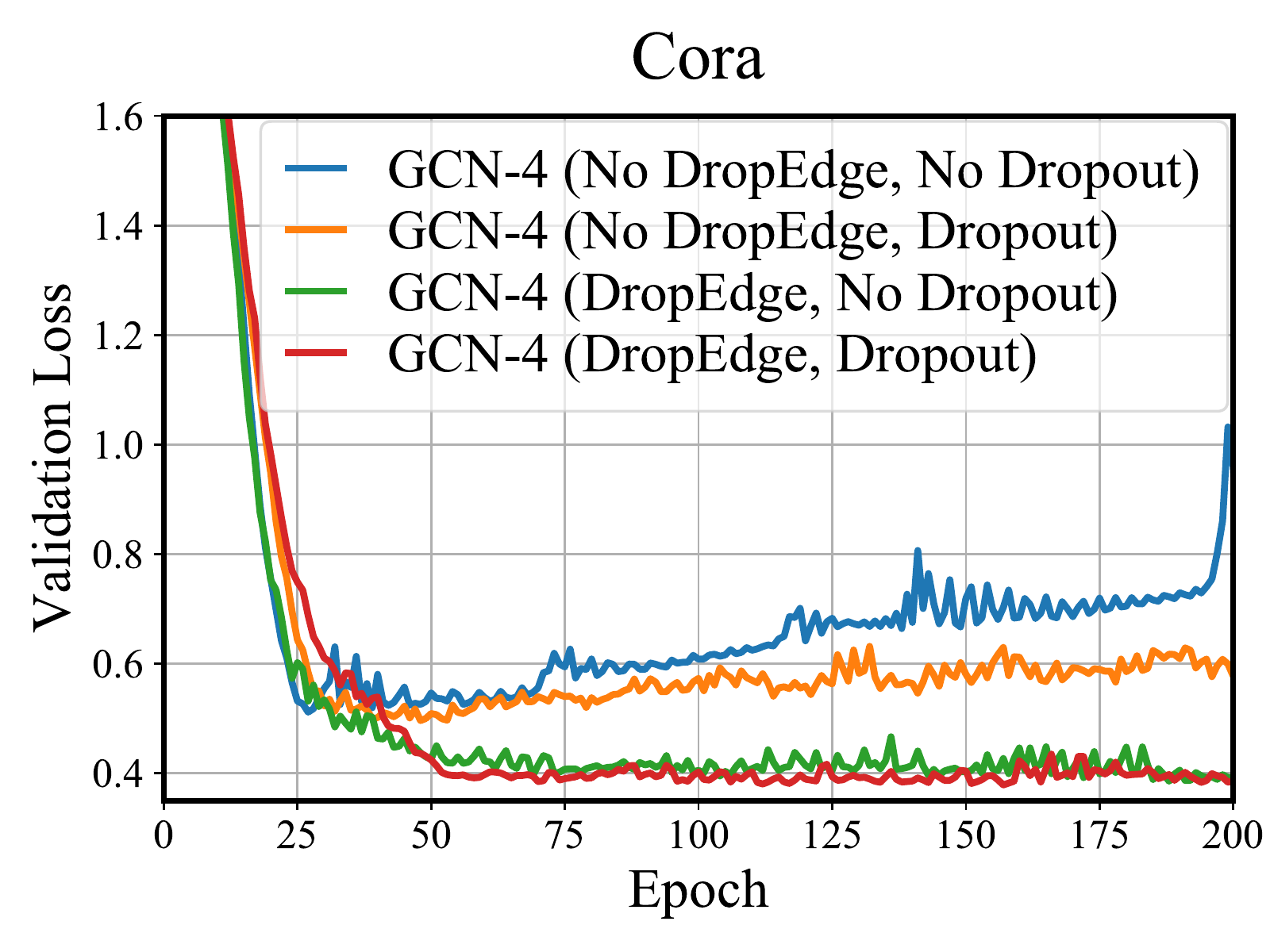}}
    \subfloat[]{
    \includegraphics [width=0.24\textwidth]{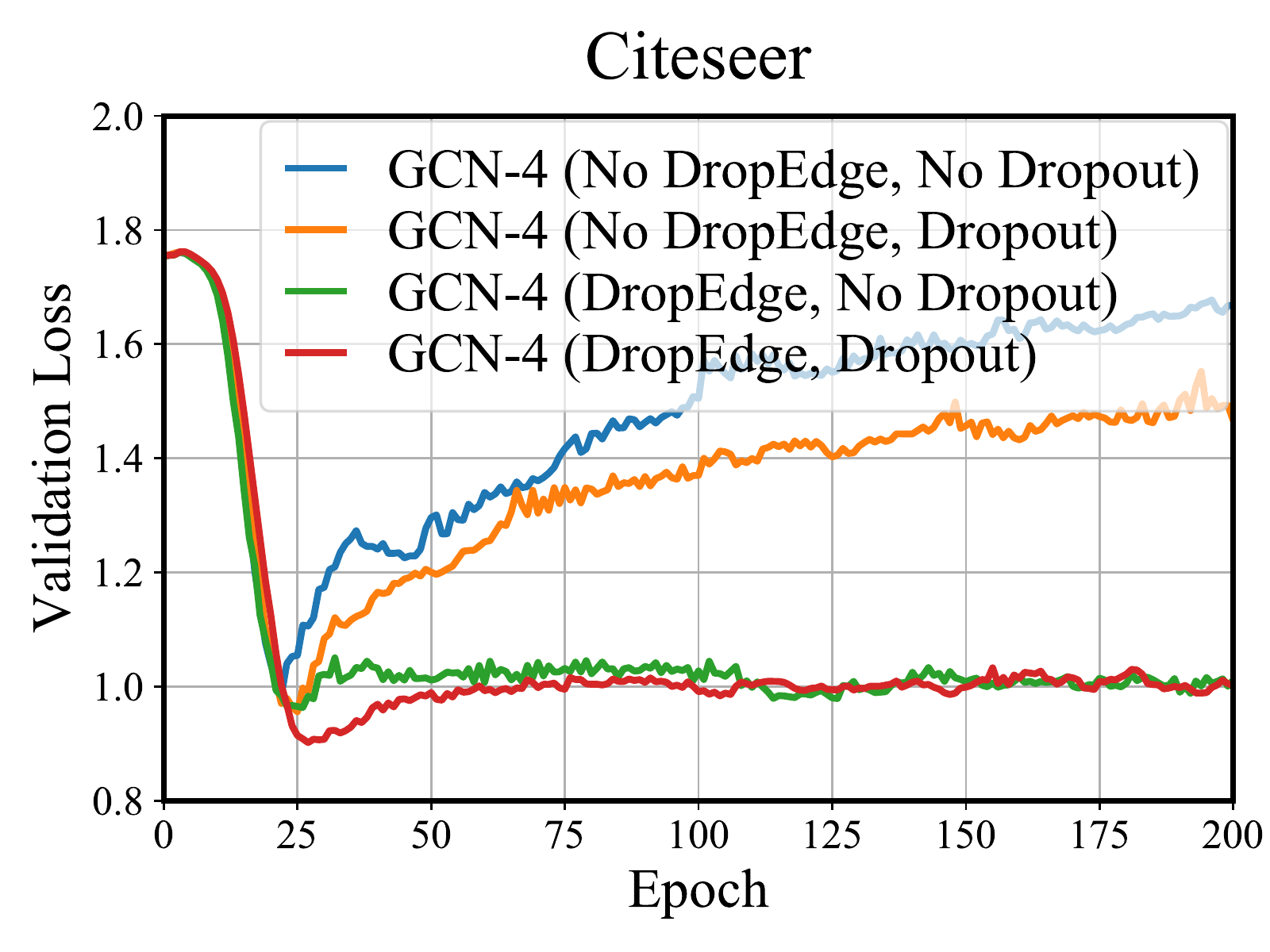}}
    \subfloat[]{
    \includegraphics[width=0.24\textwidth]{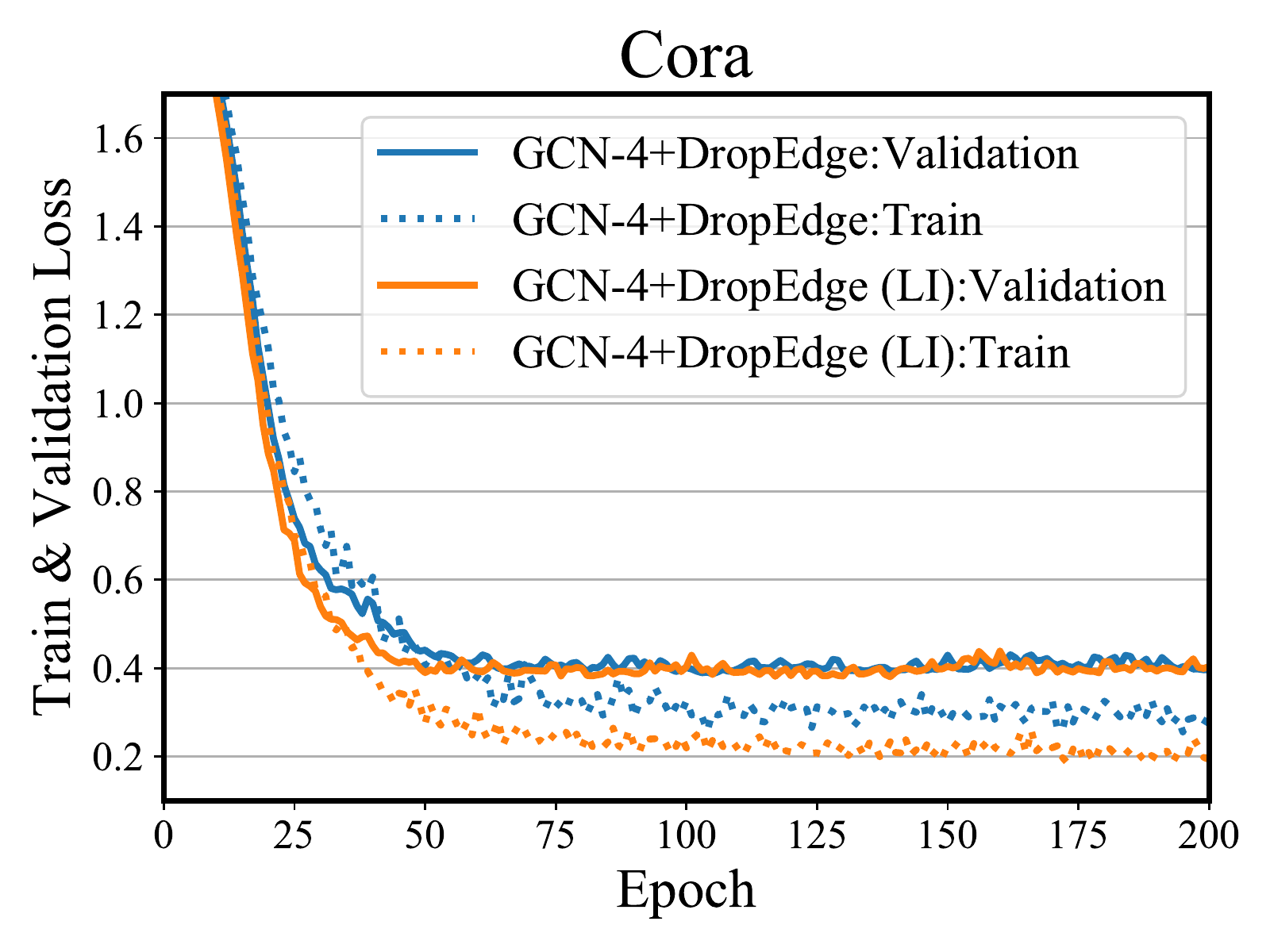}}
    \subfloat[]{
    \includegraphics [width=0.24\textwidth]{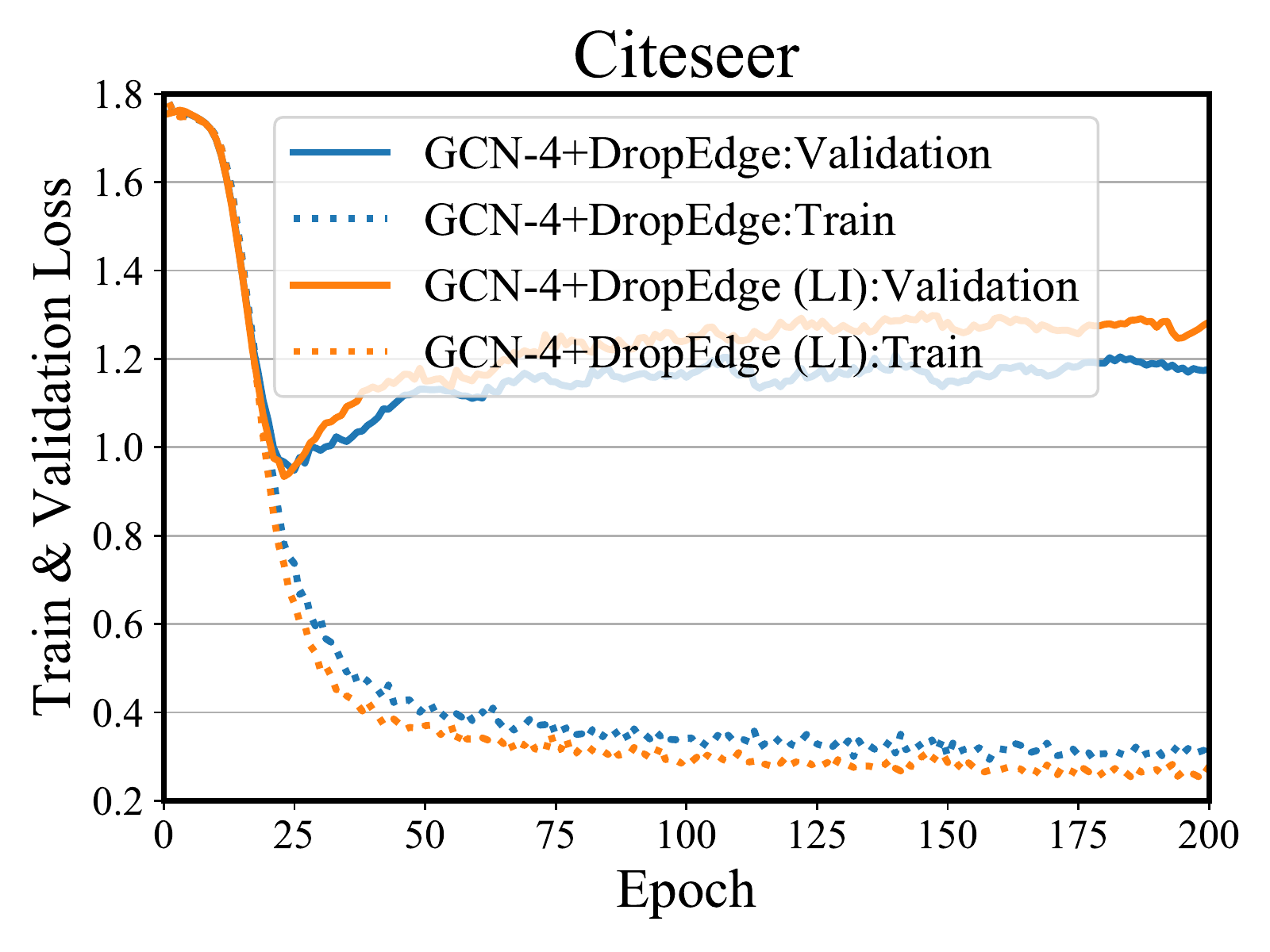}}
    \caption{(a-b) The compatibility of DropEdge with Dropout; (c-d) The performance of Layer-wise DropEdge.}
    \label{fig:more-ablation}
\end{figure*}

\subsection{Ablation Studies}
\label{sec:exp-dropedge}
This section continues several ablation studies to evaluate the importance of each proposed component in DropEdge. We employ GCN as the backbone. The hidden dimension, learning rate and weight decay are fixed to 256, 0.005 and 0.0005, receptively. The random seed is fixed. 
Unless otherwise mentioned, we do not utilize the ``withloop'' and ``withbn'' operation (see their definitions in Tab.~7 in Appendix~D).

\subsubsection{On Compatibility with Dropout}
\textsection~\ref{sec.discussions} has discussed the difference between DropEdge and Dropout. Hence, we conduct an ablation study on GCN-4, and the validation losses are demonstrated in Figure~\ref{fig:more-ablation} (a-b). It reads that while both Dropout and DropEdge are able to facilitate the training of GCN, the improvement by DropEdge is more significant, and if we adopt them concurrently, the loss is decreased further, indicating the compatibility of DropEdge with Dropout.

\subsubsection{On Layer-wise DropEdge}
\textsection~\ref{sec:methodology} has descried the Layer-Wise (LW) extension of DropEdge. Here, we provide the experimental evaluation on assessing its effect. As observed from Figure~\ref{fig:more-ablation} (c-d), the LW DropEdge achieves lower training loss than the original version, whereas the validation value between two models is comparable. It implies that LW DropEdge can facilitate the training further than original DropEdge. However, we prefer to use DropEdge other than the LW variant so as to not only avoid the risk of over-fitting but also reduces computational complexity since LW DropEdge demands to sample each layer and spends more time.

\begin{table}[htbp]
  \centering
  \caption{The performance comparison of ER-Graph, GLASSO (GGL) and DropEdge.}
  \setlength{\tabcolsep}{1pt}
    \begin{tabular}{cccccc}
    \hline
    \multicolumn{1}{l}{Dataset} & Backbone   & \multicolumn{1}{l}{Original} & \multicolumn{1}{l}{ER-Graph} & \multicolumn{1}{l}{GLASSO} & \multicolumn{1}{l}{DropEdge} \bigstrut\\
    \hline
    {Cora} & GCN-b      & 0.831 & 0.319 & 0.432 & 0.849 \bigstrut\\ 
    \hline
    {Citeseer} & GCN-b    & 0.715 & 0.233  & 0.220 & 0.763 \bigstrut\\ 
    \hline
    {Pubmed} & GCN-b    & 0.850 & 0.407 & -- & 0.861 \bigstrut\\ 
    \hline
    \end{tabular}%

  \label{tab:ergraph}%
\end{table}%

\begin{table}[htbp]
  \centering
  \caption{The running time of GLASSO (GGL) and DropEdge.}
    \begin{tabular}{ccrrr}
    \hline
    \multicolumn{1}{l}{Dataset} & Backbone & \multicolumn{1}{l}{Original} & \multicolumn{1}{l}{GLASSO} & \multicolumn{1}{l}{DropEdge} \bigstrut\\
    \hline
    {Cora} & GCN-b    & 7.77s  & 95.27s & 8.73s \bigstrut\\
    \hline
    {Citeseer} & GCN-b  & 13.49s & 328.63s   & 15.71s \bigstrut\\
    \hline
    {Pubmed} & GCN-b  & 33.26s & >40h & 35.98s \bigstrut\\
    \hline
    \end{tabular}%
  \label{tab:glassotime}%
\end{table}%

{
\subsubsection{On Comparisons with ER-Graph and GLASSO}
\label{sec:glasso}

We have conducted an ablation study of using random graphs created by the Erdos-Renyi model~\cite{erdos1960evolution} with the number of edges equal to the graph after DropEdge. We train the model GCN-b with 6 layers on these random graphs following the same setting as DropEdge for fair comparisons, and summarize the averaged results over 20 runs in Table~\ref{tab:ergraph}. It shows that all results are much worse than DropEdge probably because the generated graphs by the ER method are inconsistent to the generic statistics of edges in the original graph. On the contrary, the performance by DropEdge is always promising.

In addition, we have provided the experimental comparisons between DropEdge and GLASSO~\cite{friedman2008sparse} in Table~\ref{tab:ergraph}. For GLASSO, we compute the empirical covariance matrix based on node feature vectors, and leverage the source code\footnote{\url{https://github.com/STAC-USC/Graph_Learning}} built by [2] to perform constrained GLASSO. In particular, we select the GGL setting (Problem 1 in [2]) to enforce the off-diagonal elements of the learned graph Laplacian to be non-positive, in order to output a positive graph as expected. To keep fair comparisons, we choose appropriate hyper-parameters (my\_eps\_outer = 1.00E-05; alpha=4.00E-05 on Cora, alpha=2.00E-06 on Citeseer) to ensure that the number of edges by GGL is equal to that yielded by DropEdge. We then carry out GCN on the learned sparse graphs. From Table~\ref{tab:ergraph}, GGL gets much worse performance than DropEdge and the full-graph version on the Cora and Citeseer datasets. GGL is unable to maintain the whole information of edge connections due to the lack of the information from the original adjacency matrix, while DropEdge is capable of alleviating over-smoothing by edge sampling and still preserve the whole information during the entire training phase. Another drawback of GLASSO is that the optimization process for sparsification is time-consuming. Table~\ref{tab:glassotime} displays the training time over 400 epochs. It is observed that GLASSO involves a large amount of extra running time, making it unpractical for large-scale datasets; we can not obtain reasonable result by conducting GGL on the Pubmed dataset ($\sim$ 20,000 nodes) even after 40-hour computation.  

}

\begin{table}[htbp]
  \centering
  \caption{ The performance comparison of uniform dropping and feature weighted dropping.}
    \begin{tabular}{cccc}
    \hline
     Dataset     &  Backbone     & \multicolumn{1}{l}{Feature Weighted} & \multicolumn{1}{l}{Uniform} \bigstrut\\
    \hline
    {Cora} & GCN-b   & 0.856 & 0.876 \bigstrut\\
    \hline
    {Citeseer} & GCN-b   & 0.797 & 0.792 \bigstrut\\
    \hline
    {Pubmed} & GCN-b   & 0.888 & 0.913 \bigstrut\\
    \hline
    \end{tabular}%
  \label{tab:feaweighted}%
\end{table}%

{
\subsubsection{On Comparisons with the Feature-based Sampling}

It is possible to take the pairwise correlations/similarities into account in DropEdge. However, this will make DropEdge biased, focusing more on node feature correlations but less on authentic edge connections provided in the adjacency matrix $\mA$. As already explained before, we suppose the sampling in DropEdge to be unbiased so as to keep the original distribution of edge connections. If using the biased version, it will somehow pollute the information in $\mA$ and could cause performance detriment. To show this, we have implemented the experiment on GCN-b with 4 layers, by first computing the similarity score between node features via the RBF kernel and then dropping edges with the probability negatively related with the similarity score. All other settings keep the same as our previous experiments in Section 5. As observed from Table~\ref{tab:feaweighted}, the biased version performs slightly better than our unbiased method only  on Citeseer, but it is always worse in other cases. This observation supports the superiority of the unbiased sampling in general. We really appreciate the advice by the reviewer and have contained the corresponding discussions in Section 5.3 to make our paper more enhanced.  

}

\section{Conclusion}
{
We have analyzed the universal process of over-smoothing for 4 popular GCN models, including generic GCN, GCN with bias, ResGCN, and APPNP. Upon our analyses, we propose DropEdge, a novel and efficient technique to facilitate the development of general GCNs. Considerable experiments on Cora, Citeseer, Pubmed and Reddit have verified that DropEdge can generally and consistently promote the performance of current popular GCNs. 
Note that the experimental improvement from shallow 2-layer models to deep models (with DropEdge) is not big, but still meaningful. For example, on Pubmed, 8-layer GCN with DropEdge (91.0\%) outperforms 2-layer GCN without DropEdge (89.5\%) by 1.5\%, which is considered as being remarkable for this benchmark. More importantly, besides the experimental improvement, we have theoretically explained why deep GCN fails and how over-smoothing happens for general GCN models. 
Even we have not achieved sufficient benefits from deep GCNs as we did in other domains such as CNNs on images, both the theoretical analyses and experimental evaluations in this paper are still valuable to facilitate a broader class of future work in graph learning.
}
\ifCLASSOPTIONcaptionsoff
  \newpage
\fi



\bibliographystyle{IEEEtran}
\bibliography{IEEEabrv,ref}
%



%

\begin{IEEEbiography}[{\includegraphics[width=1in,height=1.25in,clip,keepaspectratio]{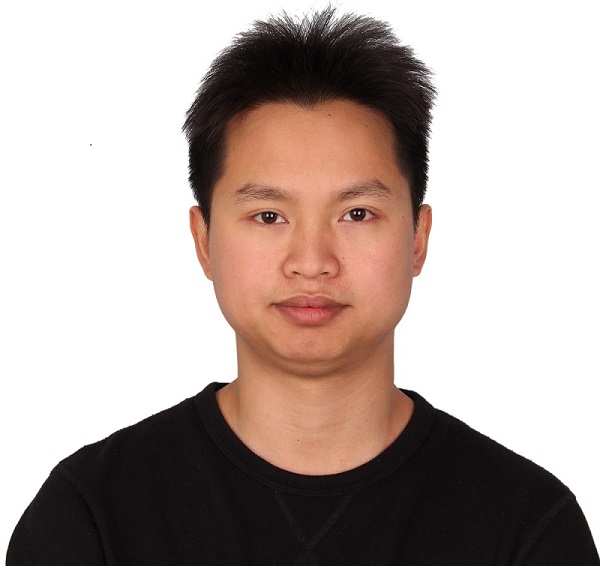}}]{Wenbing Huang} is now an assistant researcher at Department of Computer Science and Technology, Tsinghua University. He received his Ph.D. degree of computer science and technology from Tsinghua University in 2017. His current research mainly lies in the areas of machine learning, computer vision, and robotics, with particular focus on learning on irregular structures including graphs and videos. He has published about 30 peer-reviewed top-tier conference and journal papers, including the Proceedings of NeurIPS, ICLR, ICML, CVPR, etc. He served (will serve) as a Senior Program Committee of AAAI 2021, Area Chair of ACMMMM workshop HUMA 2020, and Session Chair of IJCAI 2019.
\end{IEEEbiography}

\begin{IEEEbiography}[{\includegraphics[width=1in,height=1.25in,clip,keepaspectratio]{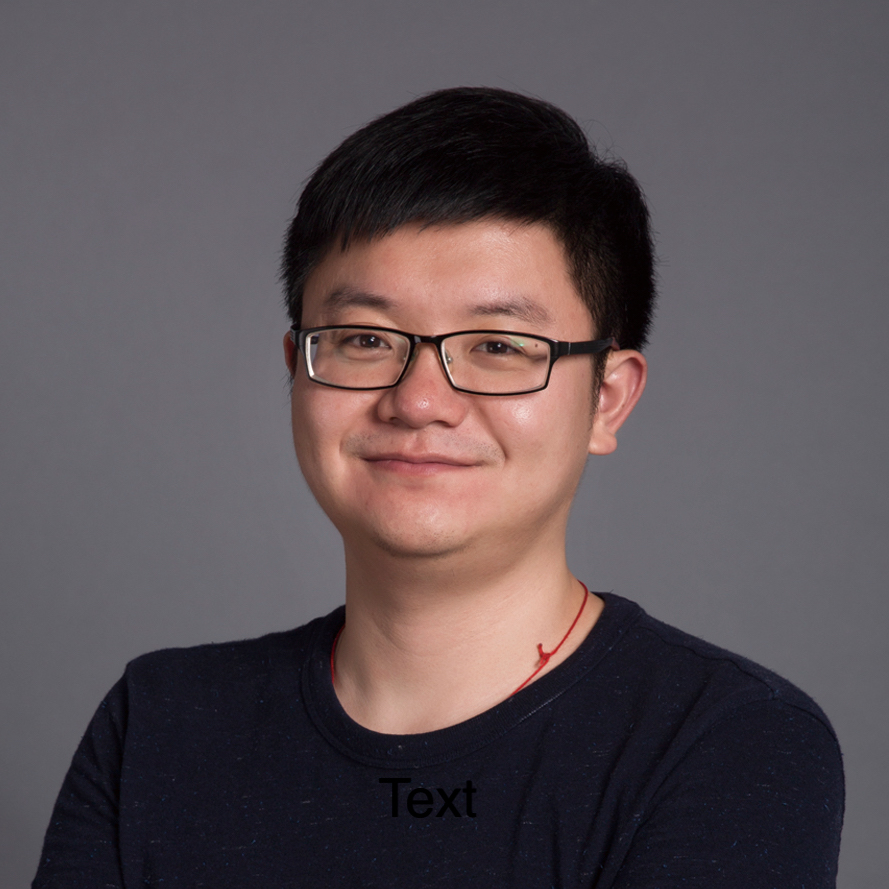}}]{Yu Rong} is a Senior researcher of Machine Learning Center in Tencent AI Lab. He received the Ph.D. degree from The Chinese University of Hong Kong in 2016. He joined Tencent AI Lab in June 2017.  His main research interests include social network analysis, graph neural networks, and large-scale graph systems. In Tencent AI Lab, he is working on building the large-scale graph learning framework and applying the deep graph learning model to various applications, such as ADMET prediction and malicious detection. He has published several papers on data mining, machine learning top conferences, including the Proceedings of KDD, WWW, NeurIPS, ICLR, CVPR, ICCV, etc. 
\end{IEEEbiography}

\begin{IEEEbiography}[{\includegraphics[width=1in,height=1.25in,clip,keepaspectratio]{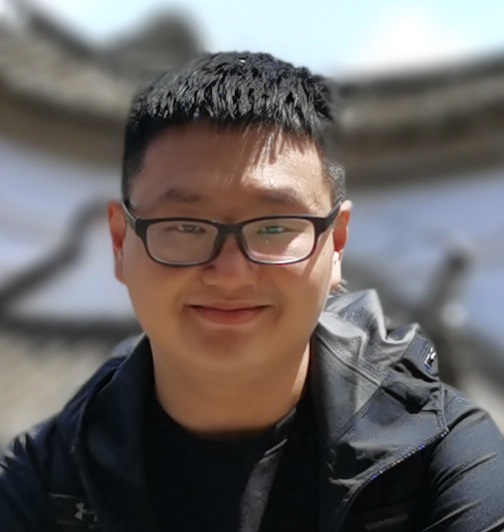}}]{Tingyang Xu} is a Senior researcher of Machine Learning Center in Tencent AI Lab. He obtained the Ph.D. degree from The University of Connecticut in 2017 and joined Tencent AI Lab in July 2017. In Tencent AI Lab, he is working on deep graph learning, graph generations and applying the deep graph learning model to various applications, such as molecular generation and rumor detection. His main research interests include social network analysis, graph neural networks, and graph generations, with particular focus on design deep and complex graph learning models for molecular generations. He has published several papers on data mining, machine learning top conferences KDD, WWW, NeurIPS, ICLR, CVPR, ICML, etc.
\end{IEEEbiography}

\begin{IEEEbiography}[{\includegraphics[width=1in,height=1.25in,clip,keepaspectratio]{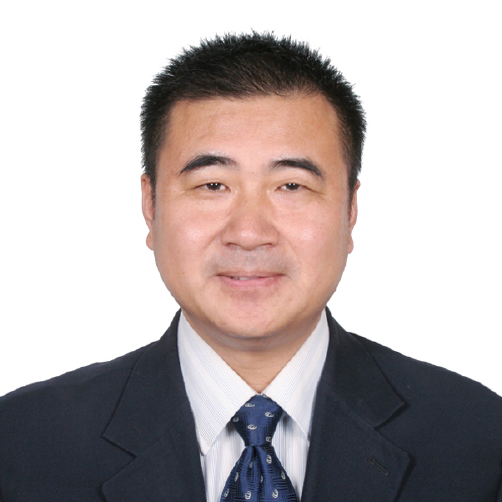}}]{Fuchun Sun}, IEEE Fellow, received the Ph.D. degree in computer science and technology from Tsinghua University, Beijing, China, in 1997. He is currently a Professor with the Department of Computer Science and Technology and President of Academic Committee of the Department, Tsinghua University, deputy director of State Key Lab. of Intelligent Technology and Systems, Beijing, China. His research interests include intelligent control and robotics, information sensing and processing in artificial cognitive systems, and networked control systems. He was recognized as a Distinguished Young Scholar in 2006 by the Natural Science Foundation of China. He became a member of the Technical Committee on Intelligent Control of the IEEE Control System Society in 2006. He serves as Editor-in-Chief of International Journal on \textit{Cognitive Computation and Systems}, and an Associate Editor for a series of international journals including the IEEE \textsc{Transactions on Cognitive and Developmental Systems}, the IEEE \textsc{Transactions on Fuzzy Systems}, and the IEEE \textsc{Transactions on Systems, Man, and Cybernetics: Systems}.
\end{IEEEbiography}

\begin{IEEEbiography}[{\includegraphics[width=1in,height=1.25in,clip,keepaspectratio]{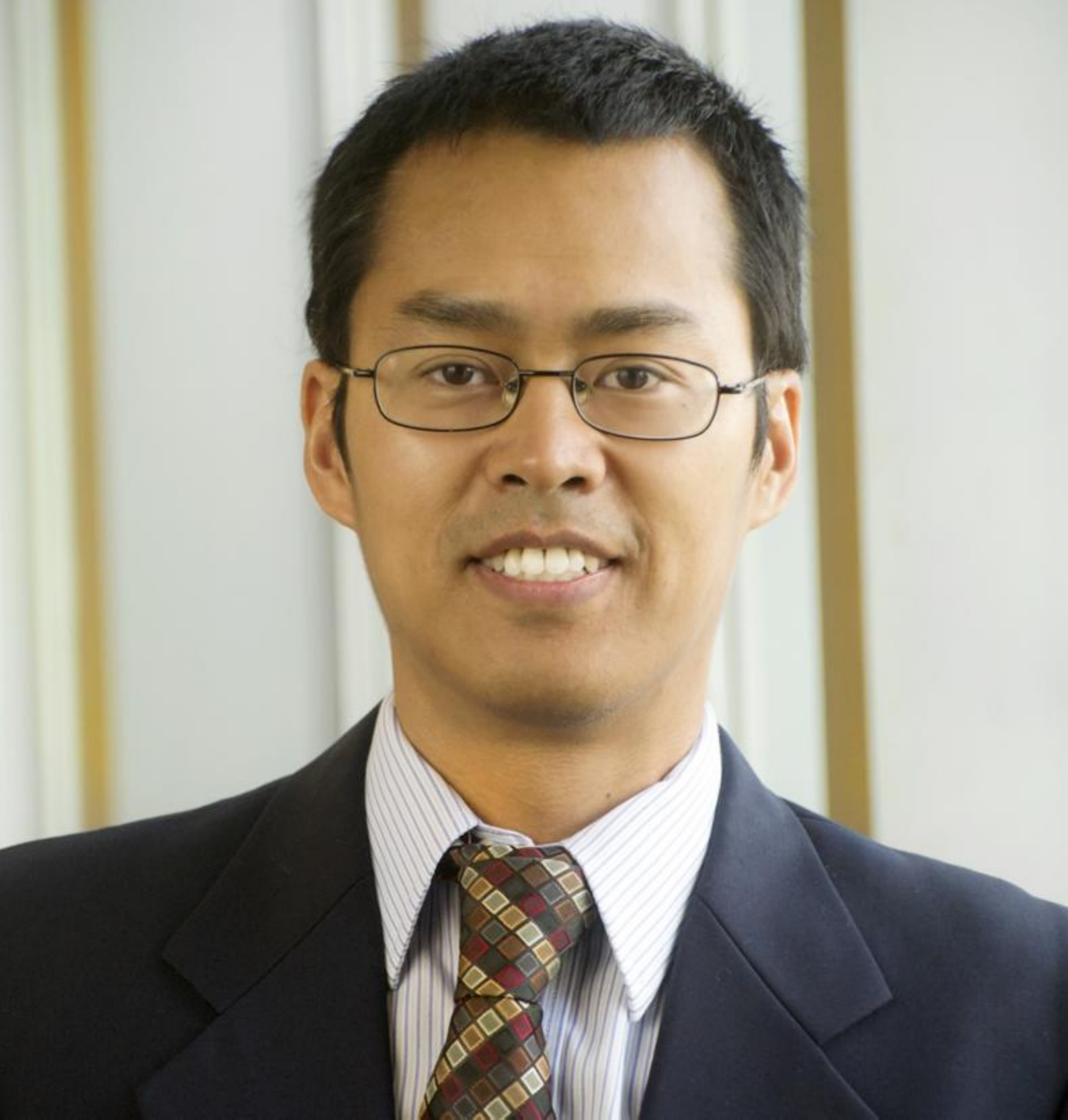}}]{Junzhou Huang}
is an Associate Professor in the Computer Science and Engineering department at the University of Texas at Arlington. He also served as the director of machine learning center in Tencent AI Lab. He received the B.E. degree from Huazhong University of Science and Technology, Wuhan, China, the M.S. degree from the Institute of Automation, Chinese Academy of Sciences, Beijing, China, and the Ph.D. degree in Computer Science at Rutgers, The State University of New Jersey. His major research interests include machine learning, computer vision and imaging informatics. He was selected as one of the 10 emerging leaders in multimedia and signal processing by the IBM T.J. Watson Research Center in 2010. His work won the MICCAI Young Scientist Award 2010, the FIMH Best Paper Award 2011, the MICCAI Young Scientist Award Finalist 2011, the STMI Best Paper Award 2012, the NIPS Best Reviewer Award 2013, the MICCAI Best Student Paper Award Finalist 2014 and the MICCAI Best Student Paper Award 2015. He received the NSF CAREER Award in 2016. 
\end{IEEEbiography}









\end{document}


%
\title{Appendices for Tackling Over-Smoothing for General Graph Convolutional Networks}
%
%
%
%

\author{Wenbing Huang$^\ast$, Yu Rong$^\ast$, Tingyang Xu, Fuchun Sun, Junzhou Huang
\IEEEcompsocitemizethanks{
\IEEEcompsocthanksitem Wenbing Huang (hwenbing@126.com) and Fuchun Sun (fcsun@mail.tsinghua.edu.cn) are with Beijing National Research Center for Information Science and Technology (BNRist), State Key Lab on Intelligent Technology and Systems, Department of Computer Science and Technology, Tsinghua University. \protect\\
\IEEEcompsocthanksitem Yu Rong (yu.rong@tencent.com), Tingyang Xu (tingyangxu@tencent.com), and Junzhou Huang (joehhuang@tencent.com) are with Tencent AI Lab, Shenzhen, China.\protect\\
\IEEEcompsocthanksitem Wenbing Huang and Yu Rong contributed to this work equally. \protect \\
\IEEEcompsocthanksitem Fuchun Sun is the corresponding author. \protect \\
}
}

%
%

\markboth{IEEE Transactions on Pattern Analysis and Machine Intelligence,~Vol.~14, No.~8, August~2015}%
{Shell \MakeLowercase{\textit{et al.}}: Bare Demo of IEEEtran.cls for Computer Society Journals}
\maketitle

\IEEEdisplaynontitleabstractindextext

%
\IEEEpeerreviewmaketitle

\appendices
\section{Proof of Theorem 2}
\label{appendix:theorem2}
\setcounter{equation}{7}
\begin{lemma}
\label{le:property}
For any $\mH,\mB\in\R^{N\times C}$ and $\alpha_1, \alpha_2\geq0$, we have:
\begin{eqnarray}
\label{eq:A}
d_{\gM}(\hat{\mA}\mH)&\leq& \lambda d_{\gM}(\mH),\\
\label{eq:W}
d_{\gM}(\mH\mW)&\leq& sd_{\gM}(\mH),\\
\label{eq:relu}
d_{\gM}(\sigma(\mH))&\leq& d_{\gM}(\mH),\\
\label{eq:bias}
d_{\gM}(\alpha_1\mH+\alpha_2\mB)&\leq& \alpha_1 d_{\gM}(\mH)+\alpha_2 d_{\gM}(\mB),
\end{eqnarray}
where $\sigma$ is ReLU function, and the denotations of $\lambda, s, \gM$ follow Theorem 2.
\end{lemma}
\begin{proof}

Oono \& Suzuki~\cite{oono2019asymptotic} have proved the first three inequalities. Their proof is based on eigen-decomposition with Kronecker product, which is sort of tedious. Here, we additionally discuss Ineq.~\ref{eq:bias}, and prove all the four inequalities in a new and concise way. 

Our proof is mainly based on the notion of projection~\cite{horn2012matrix} that returns the projected vector/matrix onto a subspace from any given vector/matrix. In terms of the subspace $\gM$, the projection matrix is given by $\hat{\mE}\hat{\mE}^{\mathrm{T}}$, where $\hat{\mE}$ is the normalized bases of the subspace $\gM$ defined in Definition 1. We also define the orthogonal complement of $\hat{\mE}$ as $\hat{\mF}$. Then, the distance $d_{\gM}(\mH)$ of arbitrary $\mH$ is derived as 
\begin{eqnarray}
\label{eq:dM-1}
d_{\gM}(\mH) &=& \|(\mI-\hat{\mE}\hat{\mE}^{\mathrm{T}})\mH\|_F\\
\nonumber
&=& \|\hat{\mF}\hat{\mF}^{\mathrm{T}}\mH\|_F\\
\nonumber
&=& \text{tr}\left((\hat{\mF}\hat{\mF}^{\mathrm{T}}\mH)^{\mathrm{T}}(\hat{\mF}\hat{\mF}^{\mathrm{T}}\mH)\right)\\
\nonumber
&=& \text{tr}\left((\hat{\mF}^{\mathrm{T}}\mH)^{\mathrm{T}}(\hat{\mF}^{\mathrm{T}}\mH)\right) \\
\label{eq:dM-F}
&=& \|\hat{\mF}^{\mathrm{T}}\mH\|_F.
\end{eqnarray}

With Eq.~\ref{eq:dM-F} at hand, we justify Ineq.~\ref{eq:A},~\ref{eq:W}, and~\ref{eq:bias} by
\begin{eqnarray}
\nonumber
d_{\gM}(\hat{\mA}\mH) &=& \|\hat{\mF}^{\mathrm{T}}(\hat{\mA}\mH)\|_F \\
\nonumber
&=& \|\hat{\mF}^{\mathrm{T}}(\hat{\mE}\hat{\mE}^{\mathrm{T}}+\hat{\mF}\Lambda\hat{\mF}^{\mathrm{T}})\mH\|_F \\
\nonumber
&=& \|\Lambda\hat{\mF}^{\mathrm{T}}\mH\|_F \\
\nonumber
&\leq& \sigma_{\text{max}}(\Lambda)\|\hat{\mF}^{\mathrm{T}}\mH\|_F \\
&\leq& \lambda d_{\gM}(\mH),
\end{eqnarray}
where we have applied the fact $\hat{\mA}=\hat{\mE}\hat{\mE}^{\mathrm{T}}+\hat{\mF}\Lambda\hat{\mF}^{\mathrm{T}}$, and $\sigma_{\text{max}}(\cdot)$ returns the maximal sigular value of the input matrix.
\begin{eqnarray}
\nonumber
d_{\gM}(\mH\mW) &=& \|\hat{\mF}^{\mathrm{T}}(\mH\mW)\|_F \\
\nonumber
&=& \|(\hat{\mF}^{\mathrm{T}}\mH)\mW\|_F \\
\nonumber
&\leq& \sigma_{\text{max}}(\mW) d_{\gM}(\mH) \\
&=& s d_{\gM}(\mH).
\end{eqnarray}
\begin{eqnarray}
\nonumber
d_{\gM}(\alpha_1\mH+\alpha_2\mB) &=& \|\hat{\mF}^{\mathrm{T}}(\alpha_1\mH+\alpha_2\mB)\|_F \\
\nonumber
&\leq& \|\alpha_1\hat{\mF}^{\mathrm{T}}\mH\|_F +\|\alpha_2\hat{\mF}^{\mathrm{T}}\mB\|_F\\
&=& \alpha_1 d_{\gM}(\mH)+\alpha_2 d_{\gM}(\mB).
\end{eqnarray}
Notice that the above inequation can be extended for the vector $\vb\in\R^{1\times C}$ (such as the bias in GCN-b), and we define $d_{\gM}(\vb)=d_{\gM}(\mB)$ where $\mB\in\R^{N\times C}$ is broadcasted from $\vb$ in the first dimension. 

We now prove Ineq.~\ref{eq:relu}. As $\hat{\mE}$ is defined by the node indicator of connected components in Theorem~1, all elements in $\hat{\mE}$ are non-negative. Moreover, since each node can only belong to one connected component, the non-zero entries in different column $\ve_i$ of $\hat{\mE}$ are located in a non-overlap way. It means, Eq.~\ref{eq:dM-1} can be further decomposed as
\begin{eqnarray}
\label{eq:dM-2}
d_{\gM}(\mH) &=& \sum_{i=1}^M \|(\mI-\ve_i\ve_i^{\mathrm{T}})\mH_i\|_F,
\end{eqnarray}
where the $j$-th row of $\mH_i\in\R^{N\times C}$ is copied from  $\mH$ if $j$ belongs to component $i$ and is zero otherwise. Then,
\begin{eqnarray}
\label{eq:decompose}
\nonumber
d^2_{\gM}(\mH) &=& \sum_{i=1}^M \|(\mI-\ve_i\ve_i^{\mathrm{T}})\mH_i\|_F^2 \\
\nonumber
&=& \sum_{i=1}^M \text{tr}\left(\mH_i^{\mathrm{T}}(\mI-\ve_i\ve_i^{\mathrm{T}})^{2}\mH_i \right)\\
\nonumber
&=& \sum_{i=1}^M \text{tr}\left(\mH_i^{\mathrm{T}}(\mI-\ve_i\ve_i^{\mathrm{T}})\mH_i \right)\\
&=& \sum_{i=1}^M \sum_{c=1}^C \vh_{ic}^{\mathrm{T}}\vh_{ic}-(\vh_{ic}^{\mathrm{T}}\ve_i)^2,
\end{eqnarray}
where $\vh_{ic}\in\R^{N}$ denotes the $c$-th column of $\mH_i$. We further denote the non-negative and negative elements of $\vh_{ic}$ as $\vh_{ic}^{+}$ and $\vh_{ic}^{-}$. Similar to Eq.~\ref{eq:decompose}, we have
\begin{eqnarray}
\label{eq:decompose-s}
d_{\gM}^2(\sigma(\mH)) &=& \sum_{i=1}^M \sum_{c=1}^C (\vh_{ic}^{+})^{\mathrm{T}}\vh_{ic}^{+}-((\vh_{ic}^{+})^{\mathrm{T}}\ve_i)^2.
\end{eqnarray}
Then, we minus Eq.~\ref{eq:decompose-s} with Eq.~\ref{eq:decompose},
\begin{eqnarray}
\label{eq:relu-final}
\nonumber
&& d^2_{\gM}(\mH)-d^2_{\gM}(\sigma(\mH)) \\
\nonumber
&=& \sum_{i=1}^M \sum_{c=1}^C (\vh_{ic}^{-})^{\mathrm{T}}\vh_{ic}^{-}-((\vh_{ic}^{-})^{\mathrm{T}}\ve_i)^2\\
\nonumber
&& -2(\vh_{ic}^{-})^{\mathrm{T}}\ve_i(\vh_{ic}^{+})^{\mathrm{T}}\ve_i \quad (\vh_{ic}^{-}<0, \vh_{ic}^{+}\geq0, \ve_i\geq 0) \\
\nonumber
&\geq& \sum_{i=1}^M \sum_{c=1}^C (\vh_{ic}^{-})^{\mathrm{T}}\vh_{ic}^{-}-((\vh_{ic}^{-})^{\mathrm{T}}\ve_i)^2 \\
\nonumber
&\geq& \sum_{i=1}^M \sum_{c=1}^C (\vh_{ic}^{-})^{\mathrm{T}}\vh_{ic}^{-} - ((\vh_{ic}^{-})^{\mathrm{T}}\vh_{ic}^{-})(\ve_i^{\mathrm{T}}\ve_i),\\
&=& 0.
\end{eqnarray}
where the last inequation employs the Cauchy–Schwarz inequality. Hence, we have proved Ineq.~\ref{eq:relu}.
\end{proof}

Based on Lemma~\ref{le:property}, we can immediately justify Theorem~2 as follows.

For GCN in Eq.~1, we apply Ineq.~\ref{eq:A}, \ref{eq:W} and \ref{eq:relu},
\begin{eqnarray}
\label{eq:dm-gcn}
\nonumber
d_{\gM}(\mH_{l+1}) &\leq& d_{\gM}(\hat{\mA}\mH_l\mW_l) \\
\nonumber
&\leq& \lambda d_{\gM}(\mH_l\mW_l) \\
&\leq& s\lambda d_{\gM}(\mH_l). 
\end{eqnarray}

For GCN-b in Eq.~2, we apply Ineq.~\ref{eq:A}-~\ref{eq:bias},
\begin{eqnarray}
\label{eq:dm-gcn-b}
\nonumber
d_{\gM}(\mH_{l+1}) &\leq& s\lambda d_{\gM}(\mH_l)+d_{\gM}(\vb_l) 
\end{eqnarray}
\begin{eqnarray}
\nonumber
\Rightarrow & & d_{\gM}(\mH_{l+1})-\frac{d_{\gM}(\vb_l)}{1-s\lambda} \\
& \leq & s\lambda \left(d_{\gM}(\mH_l)-\frac{d_{\gM}(\vb_l)}{1-s\lambda}\right).
\end{eqnarray}

For ResGCN in Eq.~3, we apply Ineq.~\ref{eq:A}-~\ref{eq:bias},
\begin{eqnarray}
\label{eq:dm-resgcn}
\nonumber
d_{\gM}(\mH_{l+1}) &\leq&  s\lambda d_{\gM}(\mH_l)+\alpha d_{\gM}(\mH_l) \\
&=& (s\lambda+\alpha) d_{\gM}(\mH_l). 
\end{eqnarray}

For APPNP in Eq.~4, we apply Ineq.~\ref{eq:A} and \ref{eq:bias},
\begin{eqnarray}
\label{eq:dm-appnp}
\nonumber
d_{\gM}(\mH_{l+1}) &\leq& (1-\beta)\lambda d_{\gM}(\mH_l)+\beta d_{\gM}(\mH_0)
\end{eqnarray}
\begin{eqnarray}
\nonumber
\Rightarrow && d_{\gM}(\mH_{l+1})-\frac{\beta d_{\gM}(\mH_0)}{1-(1-\beta)\lambda} \\
 & \leq & (1-\beta)\lambda \left(d_{\gM}(\mH_l)-\frac{\beta d_{\gM}(\mH_0)}{1-(1-\beta)\lambda}\right).
\end{eqnarray}

Clearly, Ineq.~\ref{eq:dm-gcn}-\ref{eq:dm-appnp} imply the general form in Theorem~2.

\section{Proof of Theorem~3}
\label{appendix:theorem3}
Our proof basically explores the relationship between the Laplacian matrices of the original version and the one after DropEdge. We first provide the related notations for better readability.

\textbf{Notations.} 
We reuse the aforementioned definitions of the adjacency, the Laplacian, and the degree matrix as $\mA$, $\mL$, and $\mD$, respectively, and define these terms after DropEdge as $\mA_{\text{drop}}$, $\mL_{\text{drop}}$, and $\mD_{\text{drop}}$. We denote the re-normalized (adding self-loops) of each above symbol in a hatted form, such as $\hat{\mA}$ denoting the normalized augmented adjacency.

The expected adjacency matrix $\mA_{\text{drop}}$ by DropEdge (Eq.~6) is given by $\mA_{\text{drop}}=(1-p)\mA$. Then after re-normalization,
\begin{eqnarray}
\label{eq:re-norm-drop}
\nonumber
\hat{\mA}_{\text{drop}}&=&(\mD_{\text{drop}}+\mI)^{-\frac{1}{2}}(\mA_{\text{drop}}+\mI)(\mD_{\text{drop}}+\mI)^{-\frac{1}{2}} \\
\nonumber
&=& (\mD+\frac{\mI}{(1-p)})^{-\frac{1}{2}}(\mA+\frac{\mI}{(1-p)})(\mD+\frac{\mI}{(1-p)})^{-\frac{1}{2}} \\
&:=& \mD_p^{-\frac{1}{2}}\mA_p\mD_p^{-\frac{1}{2}},
\end{eqnarray}
where we define $\mA_p=\mA+\frac{\mI}{(1-p)}$ and its degree matrix $\mD_p=\mD+\frac{\mI}{(1-p)}$.  Eq.~\ref{eq:re-norm-drop} is indeed a general form of the re-normalization trick proposed by~\cite{Kipf2017}, where we obtain $\hat{\mA}_{\text{drop}}=\hat{\mA}$ when $p=0$ and assign more weights to the self-loops when $p>0$.  For consistent denotation in Eq.~\ref{eq:re-norm-drop}, we re-specify $\hat{\mA}_{\text{drop}}$ as $\hat{\mA}_p$ below.

We can easily check the correlation between $\hat{\mL}_p$ and $\hat{\mL}$ by:
\begin{eqnarray}
\label{eq:laplacian-drop}
\nonumber
\hat{\mL}_p&:=&\mI-\hat{\mA}_p \\
\nonumber
&=& \mD_p^{-\frac{1}{2}}(\mD_p-\mA_p)\mD_p^{-\frac{1}{2}} \\
\nonumber
&=& \mD_p^{-\frac{1}{2}}(\mD-\mA)\mD_p^{-\frac{1}{2}} \\
\nonumber
&=& \mD_p^{-\frac{1}{2}}\hat{\mD}^{\frac{1}{2}}(\mI-\hat{\mA})\hat{\mD}^{\frac{1}{2}}\mD_p^{-\frac{1}{2}} \\
&=& \mD_p^{-\frac{1}{2}}\hat{\mD}^{\frac{1}{2}}\hat{\mL}\hat{\mD}^{\frac{1}{2}}\mD_p^{-\frac{1}{2}}.
\end{eqnarray}

We now denote the eigenvalue of $\hat{\mA}_p$ as $\lambda(p)$, and its associated eigenvector as $\vx_p$. 
We immediately have
\begin{eqnarray}
\label{eq:lambda-bound}
\nonumber
\lambda(p) &=& |\frac{\vx_p^{\mathrm{T}}\hat{\mA}_p\vx_p}{\|\vx_p\|^2}| \\
\nonumber
&=& |1 - \frac{\vx_p^{\mathrm{T}}\hat{\mL}_p\vx_p}{\|\vx_p\|^2}| \\
\nonumber
&=& |1 - \frac{\vx_p^{\mathrm{T}}\mD_p^{-\frac{1}{2}}\hat{\mD}^{\frac{1}{2}}\hat{\mL}\hat{\mD}^{\frac{1}{2}}\mD_p^{-\frac{1}{2}}\vx_p}{\|\vx_p\|^2}| \quad(\text{via Eq.~\ref{eq:laplacian-drop}})\\
&=& |1- \frac{\vy_p^{\mathrm{T}}\hat{\mL}\vy_p}{\|\vy_p\|^2}\frac{\|\vy_p\|^2}{\|\vx_p\|^2}|,
\end{eqnarray}
where we set $\vy_p=\hat{\mD}^{\frac{1}{2}}\mD_p^{-\frac{1}{2}}\vx_p$. Let $a=\max_{\vy_p=\hat{\mD}^{\frac{1}{2}}\mD_p^{-\frac{1}{2}}\vx_p}|1-\frac{\vy_p^{\mathrm{T}}\hat{\mL}\vy_p}{\|\vy_p\|^2}|=\max_{\vy_p=\hat{\mD}^{\frac{1}{2}}\mD_p^{-\frac{1}{2}}\vx_p}|\frac{\vy_p^{\mathrm{T}}\hat{\mA}\vy_p}{\|\vy_p\|^2}|$. Clearly, $0\leq a\leq 1$, $0\leq \frac{\|\vy_p\|^2}{\|\vx_p\|^2}=\frac{\|\hat{\mD}^{\frac{1}{2}}\mD_p^{-\frac{1}{2}}\vx_p\|^2}{\|\vx_p\|^2} \leq 1$.

Hence, according to Eq.~\ref{eq:lambda-bound}, we arrive at
\begin{eqnarray}
\label{eq:final-bound}
\nonumber
\lambda(p) &=& |1-\frac{\vy^{\mathrm{T}}\hat{\mL}\vy}{\|\vy\|^2}\frac{\|\vy_p\|^2}{\|\vx_p\|^2}| \\
\nonumber
&=& |1-\frac{\|\vy_p\|^2}{\|\vx_p\|^2}+(1-\frac{\vy^{\mathrm{T}}\hat{\mL}\vy}{\|\vy\|^2})\frac{\|\vy_p\|^2}{\|\vx_p\|^2}| \\
\nonumber
&\leq& 1-\frac{\|\vy_p\|^2}{\|\vx_p\|^2}+a\frac{\|\vy_p\|^2}{\|\vx_p\|^2} \quad( |x+y|\leq|x|+|y|) \\
\nonumber
&=& 1- (1-a)\frac{\|\vy_p\|^2}{\|\vx_p\|^2} \\
\nonumber
&\leq& 1- (1-a)\min_{d_i}\frac{d_i+1}{d_i+1/(1-p)} \\
&:=& \gamma(p).
\end{eqnarray}

On the contrary, 
\begin{eqnarray}
\nonumber
\lambda(p) &=& |1-\frac{\|\vy_p\|^2}{\|\vx_p\|^2}+(1-\frac{\vy^{\mathrm{T}}\hat{\mL}\vy}{\|\vy\|^2})\frac{\|\vy_p\|^2}{\|\vx_p\|^2}| \\
\nonumber
&\geq& 1-\frac{\|\vy_p\|^2}{\|\vx_p\|^2}-a\frac{\|\vy_p\|^2}{\|\vx_p\|^2} \quad( |x+y|\geq|x|-|y|) \\
\nonumber
&=&  1-(1+a)\frac{\|\vy_p\|^2}{\|\vx_p\|^2} \\
\nonumber
&\geq& 1- (1+a)\max_{d_i}\frac{d_i+1}{d_i+1/(1-p)} \\
&:=& \mu(p).
\end{eqnarray}
Therefore, we have
\begin{eqnarray}
\label{eq:bound}
\mu(p)\leq\lambda\leq\gamma(p),
\end{eqnarray}
where, both $\mu(p)$ and $\gamma(p)$ monotonically increase in terms of $p$; the gap $\gamma(p)-\mu(p)=a(\max_{d_i}\frac{d_i+1}{d_i+1/(1-p)}+\min_{d_i}\frac{d_i+1}{d_i+1/(1-p)})$ monotonically decreases, and when $p=1$ (dropping all edges), $\mu(1)=\lambda(1)=\gamma(1)=1$.

Eq.~\ref{eq:bound} indicates that dropping edges probably increases the value of $\lambda(p)$ if $p$ is approaching 1. Actually, when $p$ is small, we have $a=\lim_{p\rightarrow 0}\max_{\vy_p}|\frac{\vy_p^{\mathrm{T}}\hat{\mA}\vy_p}{\|\vy_p\|^2}|=\lambda(0)$. Hence for small $\lambda$, the gap between $\gamma(p)$ and $\mu(p)$ is also small, and the monotonically-increasing property of $\lambda(p)$ w.r.t. $p$ holds more potentially.

\section{Proof of Theorem~4}

The proof is straightforward, since the number of connected components only increases if the edges connecting two different components are dropped.

\section{Models and Backbones}
\label{appendix:models}
\textbf{Backbones}
We employ one input GCL and one output GCL on ResGCN, APPNP, and JKNet. Therefore, the layers in ResGCN, APPNP and JKNet are at least 3 layers. 
All backbones are implemented in Pytorch~\cite{paszke2017automatic}.

\textbf{Self Feature Modeling}
We also implement a variant of graph convolution layer with self feature modeling \cite{fout2017protein}:
\begin{align}
    \mathbf{H}_{l+1} = \sigma\left(\hat{\mathbf{A}}\mathbf{H}_{l}\mathbf{W}_{l} + \mathbf{H}_{l}\mathbf{W}_{{\text{self}}_{l}}\right),
\end{align}
where $\mathbf{W}_{{\text{self}}_{l}}\in \mathbb{R}^{C_l\times C_{l-1}}$.




\


\begin{table}[h!]
  \centering
  \caption{Hyper-parameter Description}
  \label{tab:hyper-desc}
    \small
    \begin{tabular}{l|l}
    \hline
    Hyper-parameter & Description \\
    \hline
    lr    & learning rate \\
    weight-decay & L2 regulation weight \\
    sampling-percent     & edge preserving percent ($1-p$) \\
    dropout & dropout rate \\
    normalization & the propagation models \cite{Kipf2017} \\
    withloop & using self feature modeling \\
    withbn & using batch normalization  \\
    \hline
    \end{tabular}%
\end{table}%

\begin{table*}[h!]
  \centering
  \caption{The normalization / propagation models}
      \vspace{-2ex}
    \scriptsize
    \begin{tabular}{l|l|l}
    \hline
    Description & Notation & $\mA_{\text{drop}}$ \\
    \hline
    First-order GCN & FirstOrderGCN & $\mI + \mD^{-1/2}\mA\mD^{-1/2}$ \\
    Augmented Normalized Adjacency & AugNormAdj & $(\mD + \mI)^{-1/2} ( \mA + \mI ) (\mD + \mI)^{-1/2}$ \\
    Augmented Normalized Adjacency with Self-loop & BingGeNormAdj & $\mI + (\mD + \mI)^{-1/2} (\mA + \mI) (\mD + \mI)^{-1/2}$ \\
    Augmented Random Walk & AugRWalk & $(\mD + \mI)^{-1}(\mA + \mI)$\\
    \hline
    \end{tabular}%
  \label{tab:normalization}%
\end{table*}%

\begin{table*}[h!]
  \centering
  \scriptsize
  \caption{The hyper-parameters of best accuracy (\%) for each backbone on all datasets.}
    \small
    \begin{tabular}{cl|r|r|p{0.6\textwidth}}
    \hline
    \multicolumn{1}{l}{Dataset} & Backbone & \multicolumn{1}{l|}{nlayers} & \multicolumn{1}{l|}{Acc.} & Hyper-parameters \\
    \hline
    \multirow{6}[10]{*}{Cora} & GCN   & 4     & 86.60 & lr:0.0005, weight-decay:1e-5, sampling-percent:0.7, dropout:0.7, normalization:FirstOrderGCN \\
    \cline{2-5} & GCN-b   & 4 &  87.60 & lr:0.010, weight-decay:5e-3, sampling-percent:0.7, dropout:0.8, normalization:FirstOrderGCN \\
\cline{2-5}          & ResGCN & 4     & 87.00  & lr:0.001, weight-decay:1e-5, sampling-percent:0.1, dropout:0.5, normalization:FirstOrderGCN \\
\cline{2-5}          & JKNet & 16    & 88.00  & lr:0.008, weight-decay:5e-4, sampling-percent:0.2, dropout:0.8, normalization:AugNormAdj \\
\cline{2-5}          & APPNP & 64     & 89.10 & lr:0.006, weight-decay:5e-5, sampling-percent:0.4, dropout:0.1, normalization:AugRWalk, alpha:0.2\\

    \hline
    \multirow{6}[10]{*}{Citeseer} & GCN   & 4     & 79.00 & lr:0.01, weight-decay:5e-4,sampling-percent:0.1, dropout:0.8, normalization:AugRWalk, withloop, withbn \\
    \cline{2-5} & GCN-b   &     4  & 79.20  & lr:0.009, weight-decay:1e-3, sampling-percent:0.05, dropout:0.8,
normalization:BingGeNormAdj, withloop, withbn \\
\cline{2-5}          & ResGCN & 16    & 79.40 &  lr:0.001, weight-decay:5e-3, sampling-percent:0.5, dropout:0.3, normalization:BingGeNormAdj, withloop \\
\cline{2-5}          & JKNet & 8     & 80.20 &  lr:0.004, weight-decay:5e-5, sampling-percent:0.6, dropout:0.3, normalization:AugNormAdj, withloop \\
\cline{2-5}          & APPNP & 64     & 81.30 & lr:0.010, weight-decay:1e-5, sampling-percent:0.8, dropout:0.8, normalization:AugNormAdj, alpha:0.4 \\

    \hline
    \multirow{5}[10]{*}{Pubmed} & GCN   & 8     & 91.00 & lr:0.006, weight-decay:5e-4,sampling-percent:0.3, dropout:0.8, normalization: AugRWalk, withloop, withbn \\
    \cline{2-5} & GCN-b  &    4  & 91.30  & lr:0.010, weight-decay:1e-3, sampling-percent:0.3, dropout:0.5, normalization:BingGeNormAdj, withloop, withbn \\
\cline{2-5}          & ResGCN & 32    & 91.10 &  lr:0.003, weight-decay:5e-5, sampling-percent:0.7, dropout:0.8, normalization:AugNormAdj, withloop, withbn \\
\cline{2-5}          & JKNet & 64    & 91.60 &  lr:0.005, weight-decay:1e-4, sampling-percent:0.5, dropout:0.8, normalization:AugNormAdj, withloop,withbn \\
\cline{2-5}          & APPNP & 4     & 90.70 &  lr:0.008, weight-decay:1e-4,sampling-percent:0.8, dropout:0.1, normalization:FirstOrderGCN, alpha:0.4 \\

    \hline
    \multirow{5}[10]{*}{Reddit} & GCN   & 8     & 96.57 &  lr:0.005, weight-decay:1e-5, sampling-percent:0.7, dropout:0.2, normalization:FirstOrderGCN, withloop, withbn \\
    \cline{2-5} & GCN-b   &   4   & 96.71  &   lr:0.005, weight-decay:1e-4, sampling-percent:0.6, dropout:0.5, normalization:AugRWalk, withloop \\
\cline{2-5}          & ResGCN & 16    & 96.48 & lr:0.009, weight-decay:1e-5, sampling-percent:0.2, dropout:0.5, normalization:BingGeNormAdj, withbn \\
\cline{2-5}          & JKNet & 8     & 97.02 & lr:0.010, weight-decay:5e-5, sampling-percent:0.6, dropout:0.5, normalization:BingGeNormAdj, withloop,withbn \\
\cline{2-5}          & APPNP & 8     & 95.85 & lr:0.004, weight-decay:1e-5, sampling-percent:0.5, dropout:0.1, normalization:AugRWalk, alpha:0.1 \\
    \hline
    \end{tabular}%
  \label{tab:hyperparameterdetails}%
\end{table*}%

\ifCLASSOPTIONcaptionsoff
  \newpage
\fi



\bibliographystyle{IEEEtran}
\bibliography{IEEEabrv,ref}

%% file: math_commands.tex
\usepackage{amsmath,amsfonts,bm}

\def\eqref#1{equation~\ref{#1}}

\def\1{\bm{1}}

\def\vb{{\bm{b}}}

\def\vd{{\bm{d}}}
\def\ve{{\bm{e}}}

\def\vh{{\bm{h}}}

\def\vu{{\bm{u}}}

\def\vx{{\bm{x}}}

\def\mA{{\bm{A}}}

\def\mC{{\bm{C}}}
\def\mD{{\bm{D}}}
\def\mE{{\bm{E}}}

\def\mH{{\bm{H}}}
\def\mI{{\bm{I}}}

\def\mL{{\bm{L}}}

\def\mP{{\bm{P}}}

\def\mW{{\bm{W}}}
\def\mX{{\bm{X}}}
\def\mY{{\bm{Y}}}

\DeclareMathAlphabet{\mathsfit}{\encodingdefault}{\sfdefault}{m}{sl}
\SetMathAlphabet{\mathsfit}{bold}{\encodingdefault}{\sfdefault}{bx}{n}

\def\gG{{\mathcal{G}}}

\def\gM{{\mathcal{M}}}
\def\gN{{\mathcal{N}}}
\def\gO{{\mathcal{O}}}

\def\sV{{\mathbb{V}}}

\newcommand{\E}{\mathbb{E}}

\newcommand{\R}{\mathbb{R}}

